\documentclass{article}

     \usepackage{color}
     \usepackage{xcolor}

\usepackage[final]{neurips_2022}

\usepackage{microtype}
\usepackage{graphicx}
\usepackage{subfigure}
\usepackage{booktabs} % for professional tables

\usepackage{multicol}

\usepackage{amsmath}
\usepackage{amssymb}
\usepackage{mathtools}
\usepackage{amsthm}
\usepackage{lipsum}

\usepackage[capitalize,noabbrev]{cleveref}

\theoremstyle{plain}
\newtheorem{theorem}{Theorem}[section]
\newtheorem{proposition}[theorem]{Proposition}

\theoremstyle{definition}
\newtheorem{definition}[theorem]{Definition}

\theoremstyle{remark}

\DeclareMathOperator{\sech}{sech}

\newcommand{\nin}{\ensuremath{n_{\text{in}}}}
\newcommand{\lth}{\ensuremath{l^{\text{th}}}}
\newcommand{\vol}{\ensuremath{\text{vol}}}
\newcommand{\distfun}{\ensuremath{\text{distance}}}
\newcommand{\calB}{\ensuremath{\mathcal B}}
\newcommand{\defeq}{\vcentcolon=}

\DeclareMathAlphabet{\mathcal}{OMS}{cmsy}{m}{n}

\title{Effects of Data Geometry in Early Deep Learning}

\author{%
  Saket Tiwari \\
  Department of Computer Science\\
  Brown University\\
  Providence, RI 02906 \\
  \texttt{saket\_tiwari@brown.edu} \\
  \And
  George Konidaris \\
  Department of Computer Science\\
  Brown University \\
  Providence, RI 02906 \\
}

\begin{document}

\maketitle

\begin{abstract}
  Deep neural networks can approximate functions on different types of data, from images to graphs, with varied underlying structure. 
 This underlying structure can be viewed as the geometry of the data manifold. 
 By extending recent advances in the theoretical understanding of neural networks, we study how a randomly initialized neural network with piece-wise linear activation splits the data manifold into \textit{regions} where the neural network behaves as a linear function. 
 We derive bounds on the density of boundary of linear regions and the distance to these boundaries on the data manifold.
 This leads to insights into the expressivity of randomly initialized deep neural networks on non-Euclidean data sets. 
 We empirically corroborate our theoretical results using a toy supervised learning problem.
 Our experiments demonstrate that number of linear regions varies across manifolds and the results hold with changing neural network architectures.
We further demonstrate how the complexity of linear regions is different on the low dimensional manifold of images as compared to the Euclidean space, using the MetFaces dataset.
\end{abstract}

\section{Introduction}\label{sec:intro}
The capacity of Deep Neural Networks (DNNs) to approximate arbitrary functions given sufficient training data in the supervised learning setting is well known \citep{Cybenko1989ApproximationBS, Hornik1989MultilayerFN, Anthony1999NeuralNL}.
Several different theoretical approaches have emerged that study the effectiveness and pitfalls of deep learning.
These studies vary in their treatment of neural networks and the aspects they study range from convergence \citep{AllenZhu2019ACT, Goodfellow2015QualitativelyCN}, generalization \citep{Kawaguchi2017GeneralizationID, Zhang2017UnderstandingDL, Jacot2018NeuralTK, Sagun2018EmpiricalAO}, function complexity \citep{Montfar2014OnTN, Mhaskar2016DeepVS}, adversarial attacks \citep{Szegedy2014IntriguingPO, Goodfellow2015ExplainingAH} to representation capacity \citep{Arpit2017ACL}.
Some recent theories have also been shown to closely match empirical observations \citep{Poole2016ExponentialEI, Hanin2019DeepRN, Kunin2020NeuralMS}.

One approach to studying DNNs is to examine how the underlying structure, or geometry, of the data interacts with learning dynamics.
The manifold hypothesis states that high-dimensional real world data typically lies on a low dimensional manifold \citep{Tenenbaum1997MappingAM, Carlsson2007OnTL, Fefferman2013TestingTM}.
Empirical studies have shown that DNNs are highly effective in deciphering this underlying structure by learning intermediate latent representations \citep{Poole2016ExponentialEI}.
The ability of DNNs to ``flatten'' complex data manifolds, using composition of seemingly simple piece-wise linear functions, appears to be unique \citep{Brahma2016WhyDL, Hauser2017PrinciplesOR}.

DNNs with piece-wise linear activations, such as ReLU \citep{Nair2010RectifiedLU}, divide the input space into linear regions, wherein the DNN behaves as a linear function \citep{Montfar2014OnTN}.
The density of these linear regions serves as a proxy for the DNN's ability to interpolate a complex data landscape and  has been the subject of detailed studies \citep{Montfar2014OnTN, Telgarsky2015RepresentationBO, Serra2018BoundingAC, Raghu2017OnTE}.
The work by \citet{Hanin2019ComplexityOL} on this topic stands out because they derive bounds on the average number of linear regions and verify the tightness of these bounds empirically for deep ReLU networks, instead of larger bounds that rarely materialize.
\citet{Hanin2019ComplexityOL} conjecture that the number of linear regions correlates to the expressive power of  randomly initialized DNNs with piece-wise linear activations.
However, they assume that the data is uniformly sampled from the Euclidean space $\mathbb R^d$, for some $d$.
By combining the manifold hypothesis with insights from \citet{Hanin2019ComplexityOL}, we are able to go further in estimating the number of linear regions and the average distance from \textit{linear boundaries}.
We derive bounds on how the geometry of the data manifold affects the aforementioned quantities.

To corroborate our theoretical bounds with empirical results, we design a toy problem where the input data is sampled from two distinct manifolds that can be represented in a closed form.
We count the exact number of linear regions and the average distance to the boundaries of linear regions on these two manifolds that a neural network divides the two manifolds into.
We demonstrate how the number of linear regions and average distance varies for these two distinct manifolds.
These results show that the number of linear regions on the manifold do not grow exponentially with the dimension of input data.
Our experiments do not provide estimates for theoretical constants, as in most deep learning theory, but demonstrate that the number of linear regions change as a consequence of these constants.
We also study linear regions of deep ReLU networks for high dimensional data that lies on a low dimensional manifold with unknown structure and how the number of linear regions vary on and off this manifold, which is a more realistic setting.
To achieve this we present experiments performed on the manifold of natural face images.
We sample data from the image manifold using a generative adversarial network (GAN) \citep{Goodfellow2014GenerativeAN} trained on the curated images of paintings.
Specifically, we generate images using the pre-trained StyleGAN \citep{Karras2019ASG, Karras2020AnalyzingAI} trained on the curated MetFaces dataset \citep{Karras2020TrainingGA}.
We generate \textit{curves} on the image manifold of faces, using StyleGAN, and report how the density of linear regions varies on and off the manifold.
These results shed new light on the geometry of deep learning over structured data sets by taking a data intrinsic approach to understanding the expressive power of DNNs.

\section{Preliminaries and Background}

Our goal is to understand how the underlying structure of real world data matters for deep learning.
We first provide the mathematical background required to model this underlying structure as the geometry of data.
We then provide a summary of previous work on understanding the approximation capacity of deep ReLU networks via the complexity of linear regions.
For the details on how our work fits into one of the two main approaches within the theory of DNNs, from the  expressive power perspective or from the learning dynamics perspective, we refer the reader to Appendix \ref{app:relwork}.

\subsection{Data Manifold and Definitions}

\begin{figure}[!!ht]
    \centering
    \includegraphics[width=.3\textwidth]{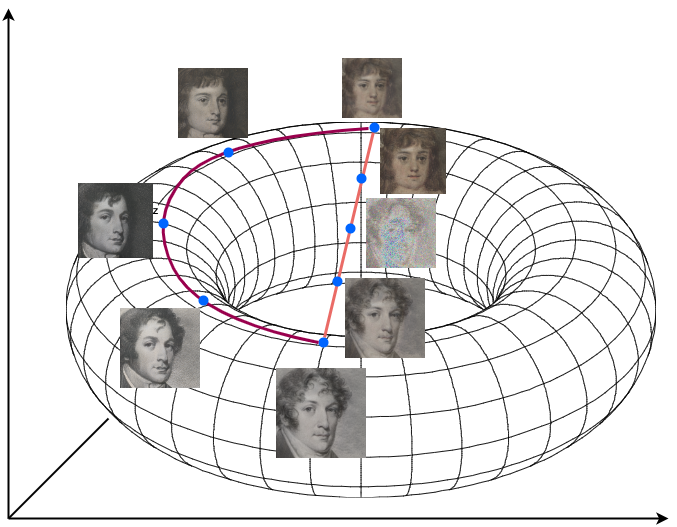}
    \caption{A 2D surface, here represented by a 2-torus, is embedded in a larger input space, $\mathbb R^3$. Suppose each point corresponds to an image of a face on this 2-torus. We can chart two curves: one straight line cutting across the 3D space and another curve that stays on the torus. Images corresponding to the points on the torus will have a smoother variation in style and shape whereas there will be images corresponding to points on the straight line that are not faces.}
    \label{fig:faces_manifold}
    \vspace{-10pt}
\end{figure}

We use the example of the MetFaces dataset \citep{Karras2020TrainingGA} to illustrate how data lies on a low dimensional manifold.
The images in the dataset are $1028 \times 1028 \times 3$ dimensional.
By contrast, the number of \textit{realistic} dimensions along which they vary are limited, e.g. painting style, artist, size and shape of the nose, jaw and eyes, background, clothing style; in fact, very few $1028 \times 1028 \times 3$ dimensional images correspond to realistic faces. 
We illustrate how this affects the possible variations in the data in Figure \ref{fig:faces_manifold}.
A manifold  formalises the notion of limited variations in high dimensional data.
One can imagine that there exists an unknown function $f: X \to Y$ from a low dimensional space of variations, to a high dimensional space of the actual data points.
Such a function $f: X \to Y$, from one open subset $X \subset \mathbb R^{m}$, to another open subset $Y \subset R^{k}$, is a \textit{diffeomorphism} if $f$ is bijective, and both $f$ and $f^{-1}$ are differentiable (or smooth). Therefore, a manifold is defined as follows.
\begin{definition}\label{def1}
\textit{
Let $k, m \in \mathbb N_0$. A subset $M \subset \mathbb R^k$ is called a smooth $m$-dimensional submanifold of $\mathbb R^k$ (or \textit{$m$-manifold in $\mathbb R^k$}) iff every point $x \in M$ has an open neighborhood $U \subset \mathbb R^k$ such that $U \cap M$ is diffeomorphic to an open subset $\Omega \subset \mathbb R^m$. 
A diffeomorphism (i.e. differentiable mapping),
\begin{equation*}
f: U \cap M \to \Omega
\end{equation*}
is called a coordinate chart of M and the inverse,
\begin{equation*}
h := f ^{-1}: \Omega \to U \cap M
\end{equation*}
is called a smooth parametrization of $U \cap M$.}
\end{definition}

For the MetFaces dataset example, suppose there are 10 dimensions along which the images vary.
Further assume that each variation can take a value continuously in some interval of $\mathbb R$.
Then the smooth parametrization would map $f: \Omega \cap \mathbb R^{10} \to M \cap \mathbb R^{1028 \times 1028 \times 3}$.
This parametrization and its inverse are unknown in general and computationally very difficult to estimate in practice.

There are similarities in how geometric elements are defined for manifolds and Euclidean spaces.
A smooth curve, on a manifold $M$, $\gamma: I \to M$ is defined from an interval $I$ to the manifold $M$ as a function that is differentiable for all $t \in I$, just as for Euclidean spaces.
The shortest such curve between two points on a manifold is no longer a straight line, but is instead  a \textit{geodesic}.
One recurring geometric element, which is unique to manifolds and stems from the definition of smooth curves, is that of a \textit{tangent space}, defined as follows.
\begin{definition}\textit{
Let $M$ be an $m$-manifold in $\mathbb R^k$ and $x \in M$ be a fixed point. 
A vector $v \in \mathbb R^k$ is called a tangent vector of $M$ at $x$ if there exists a smooth curve $\gamma: I \to M$ such that $ \gamma(0) = x, \dot{\gamma}(0) = v$ where $\dot{\gamma}(t)$ is the derivative of $\gamma$ at $t$.
The set
\begin{equation*}
T_x M := \{ \dot{\gamma}(0) | \gamma: \mathbb R \to M \text{ is smooth} \gamma(0) = x \}
\end{equation*}
of tangent vectors of $M$ at $x$ is called the tangent space of $M$ at $x$.}
\end{definition}

In simpler terms, the plane tangent to the manifold $M$ at point $x$ is called the tangent space and denoted by by $T_x M$.
Consider the upper half of a 2-sphere, $S^2 \subset \mathbb R^3$, which is a 2-manifold in $\mathbb R^3$.
The tangent space at a fixed point $x \in S^2$ is the 2D plane perpendicular to the vector $x$ and tangential to the surface of the sphere that contains the point $x$.
For additional background on manifolds we refer the reader to Appendix \ref{app:background}.

\subsection{Linear Regions of Deep ReLU Networks}

The higher the density of these linear regions the more complex a function a DNN can approximate.
For example, a $\sin$ curve in the range $[0,2\pi]$ is better approximated by 4 piece-wise linear regions as opposed to 2.
To clarify this further, with the 4 ``optimal'' linear regions $[0, \pi/2), [\pi/2, \pi), [\pi, 3\pi/2),$ and $[3\pi/2, 2\pi]$ a function could approximate the $\sin$ curve better than any 2 linear regions.
In other words, higher density of linear regions allows a DNN to approximate the variation in the curve better.
We define the notion of boundary of a linear regions in this section and provide an overview of previous results.

We consider a neural network, $F$, which is a composition of activation functions. 
Inputs at each layer are multiplied by a matrix, referred to as the weight matrix, with an additional bias vector that is added to this product.
We limit our study to ReLU activation function \citep{Nair2010RectifiedLU}, which is piece-wise linear and one of the most popular activation functions being applied to various learning tasks on different types of data like text, images, signals etc. We further consider DNNs that map inputs, of dimension $\nin$, to scalar values.
Therefore, $F: \mathbb R^{\nin} \to \mathbb R$ is defined as,
\begin{equation} \label{eq:dnntemplate}
    F(x) = W_L \sigma(B_{L - 1} + W_{L - 1} \sigma( ... \sigma(B_{1} + W_{1}x))),
\end{equation}
where $W_{l} \in \mathbb M^{n_l \times n_{l - 1}}$ is the weight matrix for the $\lth$ hidden layer, $n_l$ is the number of neurons in the $\lth$ hidden layer, $B_l \in \mathbb R^{n_l}$ is the vector of biases for the $\lth$ hidden layer, $n_0 = \nin$ and $\sigma: \mathbb R \to \mathbb R$ is the activation function.
For a neuron $z$ in the $\lth$ layer we denote the \textit{pre-activation} of this neuron, for given input $x \in \mathbb R^{\nin}$, as $z_{l}(x)$. 
For a neuron $z$ in the layer $l$ we have
\begin{equation}
\label{eq:neuroactive}
z(x) =  W_{l - 1, z} \sigma( ... \sigma(B_{1} + W_{1}x)),
\end{equation}
for $l > 1$ (for the base case $l = 1$ we have $z(x) = W_{1, z}x$) where $W_{l - 1, z}$ is the row of weights, in the weight matrix of the $\lth$ layer, $W_l$, corresponding to the neuron $z$. 
We use $W_z$ to denote the weight vector for brevity, omitting the layer index $l$ in the subscript.
We also use $b_z$ to denote the bias term for the neuron $z$.

Neural networks with piece-wise linear activations are piece-wise linear on the input space \citep{Montfar2014OnTN}.
Suppose for some fixed $y \in \mathbb R^{\nin}$ as $x \to y$ if we have $z(x) \to -b_z$ then we observe a discontinuity in the gradient $\nabla_x \sigma(b_z + W_z z(x))$ at $y$.
Intuitively, this is because $x$ is approaching the boundary of the linear region of the function defined by the output of $z$.
Therefore, the boundary of linear regions, for a feed forward neural network $F$, is defined as:
\begin{equation*}
\mathcal B_F = \{x | \nabla F(x) \text{ is not continuous at } x\}.
\end{equation*}

\citet{Hanin2019ComplexityOL} argue that an important generalization for the approximation capacity of a neural network $F$ is the $(\nin - 1)-$dimensional volume density of linear regions defined as $ \vol_{\nin - 1}(\mathcal B_F \cap K)/\vol_{\nin}(K),$ for a bounded set $K \subset \mathbb R^{\nin}$.
This quantity serves as a proxy for density of linear regions and therefore the expressive capacity of DNNs.
Intuitively, higher density of linear boundaries means higher capacity of the DNN to approximate complex non-linear functions.
The quantity is applied to lower bound the distance between a point $x \in K$ and the set $\mathcal B_F$, which is
\begin{equation*}
    \distfun(x, \mathcal B_F) = \min_{\text{neurons } z} |z(x) - b_z|/||\nabla z(x)||,
\end{equation*}
which measures the sensitivity over neurons at a given input.
The above quantity measures how ``far'' the input is from flipping any neuron from inactive to active or vice-versa.

Informally, \citet{Hanin2019ComplexityOL} provide two main results for a randomly initialized DNN $F$, with a reasonable initialisation. Firstly, they show that
\begin{equation*}
    \mathbb E \Big [ \frac{\vol_{\nin - 1}(\mathcal B_F \cap K)}{\vol_{\nin}(K)} \Big ] \approx \# \{ \text{ neurons} \},
\end{equation*}
meaning the density of linear regions is bound above and below by some constant times the number of neurons.
Secondly, for $x \in [0, 1]^{\nin}$,
\begin{equation*}
    \mathbb E \Big [ \distfun(x, \mathcal B_F) \Big ] \geq C \# \{ \text{ neurons} \}^{-1},
\end{equation*}
where $C > 0$ depends on the distribution of biases and weights, in addition to other factors. In other words, the distance to the nearest boundary is bounded above and below by a constant times the inverse of the number of neurons.
These results stand in contrast to earlier worst case bounds that are exponential in the number of neurons.
\citet{Hanin2019ComplexityOL} also verify these results empirically to note that the constants lie in the vicinity of 1 throughout training.

\section{Linear Regions on the Data Manifold}

\begin{figure}
    \centering
    \includegraphics[width=.3\textwidth]{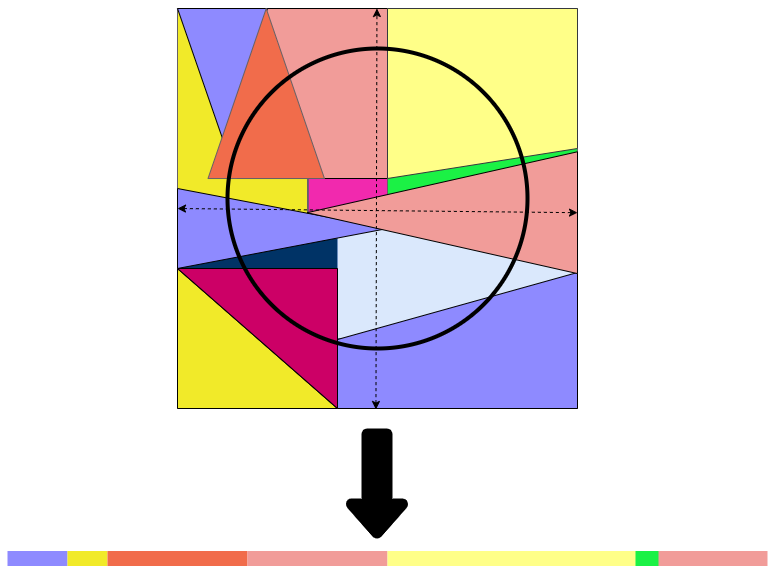}
    \caption{A circle is an example of a 1D manifold in a 2D Euclidean space. The effective number of linear regions on the manifold, the upper half of the circle, are the number of linear regions on the arc from $-\pi$ to $\pi$. In the diagram above, each color in the 2D space corresponds to a linear region. When the upper half of the circle is flattened into a 1D space we obtain a line. Each color on the line corresponds to a linear region of the 2D space.}
    \label{fig:my_label}
    \vspace{-10pt}
\end{figure}

One important assumption in the results presented by \citet{Hanin2019ComplexityOL} is that the input, $x$, lies in a compact set $K \subset \mathbb R^{\nin}$ and that $\vol_{\nin}(K)$ is greater than 0. 
Also, the theorem pertaining to the lower bound on average distance of $x$ to linear boundaries the input assumes the input uniformly distributed in $[0, 1]^{\nin}$.
As noted earlier, high-dimensional real world datasets, like images, lie on low dimensional manifolds, therefore both these assumptions are false in practice.
This motivates us to study the case where the data lies on some $m-$dimensional submanifold of $\mathbb R^{\nin}$, i.e. $M \subset \mathbb R^{\nin}$ where $m \ll \nin$.
We illustrate how this constraint effects the study of linear regions in Figure \ref{fig:my_label}.

As introduced by \citet{Hanin2019ComplexityOL}, we denote the ``$(\nin -k)-$dimensional piece'' of $\mathcal B_F$ as $\mathcal B_{F, k}$.
More precisely, $\calB_{F, 0} = \emptyset$ and $\calB_{F, k}$ is recursively defined to be the set of points $x \in \calB_F \setminus \{\calB_{F, 0} \cup ... \cup \calB_{F, k - 1} \}$ with the added condition that in a neighbourhood of $x$ the set $\calB_{F, k}$ coincides with hyperplane of dimension $\nin - k$.
We provide a detailed and formal definition for $\calB_{F, k}$ with intuition in Appendix \ref{sec:jacobproof}.
In our setting, where the data lies on a manifold $M$, we define $\calB_{F, k}'$ as $\calB_{F, k} \cap M$, and note that $\dim(\calB_{F, k}') = m - k$ (Appendix \ref{sec:jacobproof} Proposition \ref{prop:manifolddims}).
For example, the \textit{transverse} intersection (see Definition \ref{def:trans}) of a plane in 3D with the 2D manifold $S^2$ is a 1D curve in $S^2$ and therefore has dimension $1$.
Therefore, $\calB_{F, k}'$ is a submanifold of dimension $3 - 2 = 1$.
This imposes the restriction $k \leq m$, for the intersection $\calB_{F, k} \cap M$ to have a well defined volume.

We first note that the definition of the determinant of the Jacobian, for a collection of neurons $z_1, ..., z_k$, is different in the case when the data lies on a manifold $M$ as opposed to in a compact set of dimension $\nin$ in $\mathbb R^{\nin}$.
Since the determinant of the Jacobian is the quantity we utilise in our proofs and theorems repeatedly we will use the term Jacobian to refer to it for succinctness.
Intuitively, this follows from the Jacobian of a function being defined differently in the ambient space $\mathbb R^{\nin}$ as opposed to the manifold $M$.
In case of the former it is the volume of the paralellepiped determined by the vectors corresponding to the directions with steepest ascent along each one of the $\nin$ axes.
In case of the latter it is more complex and defined below.
Let $\mathcal H^m$ be the $m-$dimensional Hausdorff measure (we refer the reader to the Appendix \ref{app:background} for background on Hausdorff measure).
The Jacobian of a function on manifold $M$, as defined by \citet{Krantz2008GeometricIT} (Chapter 5), is as follows.
\begin{definition}
\label{def:jacob}
\textit{
The (determinant of) Jacobian of a function $H: M \to \mathbb R^k$, where $k \leq \dim(M) = m$, is defined as
\begin{align*}
    J_{k, H}^M(x) = \sup \Big \{& \frac{\mathcal H^k(D_M H(P))}{\mathcal H^k(P)} \Big | P \text{ is a } k \text{-dimensional parallelepiped } \text{contained in } T_{x}M.\Big \},
\end{align*}
where $D_M: T_xM \to \mathbb R^k$ is the differential map (see Appendix \ref{app:background}) and we use $D_M H(P)$ to denote the mapping of the set $P$ in $T_xM$, which is a parallelepiped, to $\mathbb R^k$. The supremum is taken over all parallelepipeds $P$.}
\end{definition}

We also say that neurons $z_1, ..., z_k$ are good at $x$ if there exists a path of neurons from $z$ to the output in the computational graph of $F$ so that each neuron is activated along the path.
Our three main results that hold under the assumptions listed in Appendix \ref{sec:assumptions}, each of which extend and improve upon the theoretical results by \citet{Hanin2019ComplexityOL}, are:

\begin{theorem}
\label{thm:jacob}
Given $F$ a feed-forward ReLU network with input dimension $\nin$, output dimension $1$, and random weights and biases.
Then for any bounded measurable submanifold $M \subset \mathbb R^{\nin}$ and any $k = 1, ...., m$ the average $(m-k)-$dimensional volume of $\calB_{F, k}$ inside $M$,
\begin{align}
\begin{split}
\label{eq:manijacobformula}
    & \mathbb E [\vol_{m - k}(\calB_{F, k} \cap M)] = \sum_{\text{distinct neurons } z_1, ..., z_k \text{ in } F} \int_{M} \mathbb E[Y_{z_1, ..., z_k}] d\vol_m(x),
\end{split}
\end{align}
where $Y_{z_1, ..., z_k}$ is $J_{m, H_k}^M(x) \rho_{b_{1}, ..., b_{k}}(z_{1}(x), ..., z_{k}(x)),$ times the indicator function of the event that $z_j$ is good at $x$ for each $j = 1, ..., k$. Here the function $\rho_{b_{z_1}, ..., b_{z_k}}$ is the density of the joint distribution of the biases $b_{z_1}, ..., b_{z_k}$.
\end{theorem}
This change in the formula, from Theorem 3.4 by \citet{Hanin2019ComplexityOL}, is a result of the fact that $z(x)$ has a different direction of steepest ascent when it is restricted to the data manifold $M$, for any $j$.
The proof is presented in Appendix \ref{sec:jacobproof}.
Formula \ref{eq:manijacobformula} also makes explicit the fact that the data manifold has dimension $m \leq \nin$ and therefore the $m - k$-dimensional volume is a more representative measure of the linear boundaries. %change this line or add reference to figure 1
Equipped with Theorem \ref{thm:jacob}, we provide a result for the density of boundary regions on manifold $M$.
\begin{theorem}
\label{thm:numneurons}
For data sampled uniformly from a compact and measurable $m$ dimensional manifold $M$ we have the following result for all $k \leq m$:
\begin{equation*}
    \frac{\text{vol}_{m - k}(\mathcal{B}_{F, k} \cap M)}{\text{vol}_{m}(M)} \leq \begin{pmatrix} \text{\# neurons} \\ k \end{pmatrix} (2C_{\text{grad}} C_{\text{bias}} C_{M})^{k},
\end{equation*}
where $C_{\text{grad}}$ depends on $||\nabla z(x)||$ and the DNN's architecture, $C_{M}$ depends on the geometry of $M$, and $C_{\text{bias}}$ on the distribution of biases $\rho_b$.
\end{theorem}

The constant $C_M$ is the supremum over the matrix norm of projection matrices onto the tangent space, $T_x M$, at any point $x \in M$.
For the Euclidean space $C_M$ is always equal to 1 and therefore the term does not appear in the work by \citet{Hanin2019ComplexityOL}, but we cannot say the same for our setting.
We refer the reader to Appendix \ref{app:proofthm2} for the proof, further details, and interpretation.
Finally, under the added assumptions that the diameter of the manifold $M$ is finite and $M$ has polynomial volume growth we provide a lower bound on the average distance to the linear boundary for points on the manifold and how it depends on the geometry and dimensionality of the manifold.

\begin{theorem}
\label{thm:dist}
For any point, $x$, chosen randomly from $M$, we have:
\begin{align*}
    \mathbb E [&\distfun_{M}(x, \calB_{F} \cap M)] \geq \frac{C_{M, \kappa}}{C_{\text{grad}} C_{\text{bias}} C_{M} \# \text{neurons}},
\end{align*}
where $C_{M, \kappa}$ depends on the scalar curvature, the input dimension and the dimensionality of the manifold $M$. The function $\distfun_M$ is the distance on the manifold $M$.
\end{theorem}

This result gives us intuition on how the density of linear regions around a point depends on the geometry of the manifold.
The constant $C_{M, \kappa}$ captures how volumes are distorted on the manifold $M$ as compared to the Euclidean space, for the exact definition we refer the reader to the proof in Appendix \ref{app:proof:thm3}.
For a manifold which has higher volume of a unit ball, on average, in comparison to the Euclidean space the constant $C_{M, \kappa}$ is higher and lower when the volume of unit ball, on average, is lower than the volume of the Euclidean space.
For background on curvature of manifolds and a proof sketch we refer the reader to the Appendices \ref{app:background} and \ref{app:proofsketch}, respectively.
Note that the constant $C_M$ is the same as in Theorem \ref{thm:numneurons}.
Another difference to note is that we derive a lower bound on the geodesic distance on the manifold $M$ and not the Euclidean distance in $\mathbb R^k$ as done by \citet{Hanin2019ComplexityOL}.
This distance better captures the distance between data points on a manifold while incorporating the underlying structure.
In other words, this distance can be understood as how much a data point should change to reach a linear boundary while ensuring that all the individual points on the curve, tracing this change, are ``valid'' data points.

\subsection{Intuition For Theoretical Results} \label{sec:intuit}

One of the key ingredients of the proofs by \citet{Hanin2019ComplexityOL} is the \textit{co-area formula} \citep{Krantz2008GeometricIT}.
The co-area formula is applied to get a closed form representation of the $k-$dimensional volume of the region where any set of $k$ neurons, $z_1, z_2, ..., z_k$ is ``good'' in terms of the expectation over the Jacobian, in the Euclidean space.
Instead of the co-area formula we use the \textit{smooth co-area formula} \citep{Krantz2008GeometricIT} to get a closed form representation of the $m-k-$dimensional volume of the region intersected with manifold, $M$, in terms of the Jacobian defined on a manifold (Definition \ref{def:jacob}).
The key difference between the two formulas is that in the smooth co-area formula the Jacobian (of a function from the manifold $M$) is restricted to the tangent plane.
While the determinant of the ``vanilla'' Jacobian measures the distortion of volume around a point in Euclidean space the determinant of the Jacobian defined as above (Definition \ref{def:jacob}) measures the distortion of volume on the manifold instead for the function with the same domain, the function that is 1 if the set of neurons are good and 0 otherwise.

The value of the Jacobian as defined in Definition \ref{def:jacob} has the same volume as the projection of the parallelepiped defined by the gradients $\nabla z(x)$ onto the tangent space (see Proposition \ref{prop:linearmap} in Appendix).
This introduces the constant $C_M$, defined above.
Essentially, the constant captures how the magnitude of the gradients, $\nabla z(x)$, are modified upon being projected to the tangent plane.
Certain manifolds ``shrink'' vectors upon projection to the tangent plane more than others, on an average, which is a function of their geometry.
We illustrate how two distinct manifolds ``shrink'' the gradients differently upon projection to the tangent plane as reflected in the number of linear regions on the manifolds (see Figure \ref{fig:projection} in the appendix) for 1D manifolds.
We provide intuition for the curvature of a manifold in Appendix \ref{app:background}, due to space constraints, which is used in the lower bound for the average distance in Theorem \ref{thm:dist}.
The constant $C_{M, \kappa}$ depends on the curvature as the supremum of a polynomial whose coefficients depend on the curvature, with order at most $\nin$ and at least $\nin - m$.
Note that despite this dependence on the ambient dimension, there are other geometric constants in this polynomial (see Appendix G).
Finally, we also provide a simple example as to how this constant varies with $\nin$ and $m$, for a simple and contrived example, in Appendix \ref{subsec:varyninm}.
\section{Experiments}

\subsection{Linear Regions on a 1D Curve}
\label{subsec:1dmethod}
To empirically corroborate our theoretical results, we calculate the number of linear regions and average distance to the linear boundary on 1D curves for regression tasks in two settings.
The first is for 1D manifolds embedded in 2D and higher dimensions and the second is for the high-dimensional data using the MetFaces dataset.
We use the same algorithm, for the toy problem and the high-dimensional dataset, to find linear regions on 1D curves.
We calculate the exact number of linear regions for a 1D curve in the input space, $x:I \to \mathbb R^{\nin}$ where $I$ is an interval in real numbers, by finding the points where $z(x(t)) = b_z$ for every neuron $z$.
The solutions thus obtained gives us the boundaries for neurons on the curve $x$.
We obtain these solutions by using the programmatic activation of every neuron and using the sequential least squares programming (SLSQP) algorithm \citep{Kraft} to solve for $|z(x(t)) - b_z| = 0$ for $t \in I$.
In order to obtain the programmatic activation of a neuron we construct a Deep ReLU network as defined in Equation \ref{eq:neuroactive}.
We do so for all the neurons for a given DNN with fixed weights.

\subsection{Supervised Learning on Toy Dataset}

\label{sec:toydataexp}

\begin{figure}
    \centering
    \subfigure[]{\includegraphics[width=0.38\textwidth]{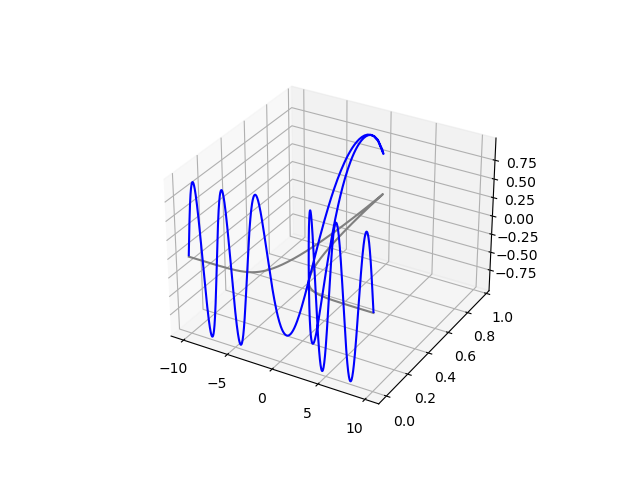}}
    \subfigure[]{\includegraphics[width=0.38\textwidth]{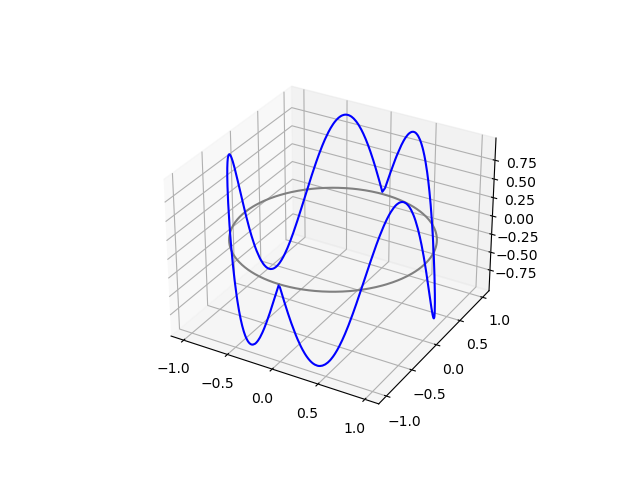}}
    \caption{The tractrix (a) and circle (b) are plotted in grey and the target function is in blue. This is for illustration purposes and does not match the actual function or domains used in our experiments.}
    \label{fig:tractrix}
    \vspace{-10pt}
\end{figure}

We define two similar regression tasks where the data is sampled from two different manifolds with different geometries.
We parameterize the first task, a unit circle without its north and south poles, by $\psi_{\text{circle}}:(-\pi, \pi) \to \mathbb R^2$ where $\psi_{\text{circle}}(\theta) = (\cos \theta, \sin \theta)$ and $\theta$ is the angle made by the vector from the origin to the point with respect to the x-axis.
We set the target function for  regression task to be a periodic function in $\theta$.
The target is defined as $z(\theta) = a \sin (\nu \theta)$ where $a$ is the amplitude and $\nu$ is the frequency (Figure \ref{fig:tractrix}).
DNNs have difficulty learning periodic functions \citep{Ziyin2020NeuralNF}.
The motivation behind this is to present the DNN with a challenging task where it has to learn the underlying structure of the data.
Moreover the DNN will have to split the circle into linear regions.
For the second regression task, a tractrix is parametrized by $\psi_{\text{tractrix}}: \mathbb R^1 \to \mathbb R^2$ where $\psi_{\text{tractrix}}(y) = (y - \tanh y, \sech y)$ (see Figure \ref{fig:tractrix}).
We assign a target function $z(t) = a \sin (\nu t)$.
For the purposes of our study we restrict the domain of $\psi_{\text{tractrix}}$ to $(-3, 3)$.
We choose $\nu$ so as to ensure that the number of peaks and troughs, 6, in the periodic target function are the same for both the manifolds.
This ensures that the domains of both the problems have length close to 6.28.
Further experimental details are in Appendix \ref{app:toyprob}.

The results, averaged over 20 runs, are presented in Figures \ref{fig:graph_lr} and \ref{fig:graph_dists}.
We note that $C_M$ is smaller for Sphere (based on Figure \ref{fig:graph_lr}) and the curvature is positive whilst $C_M$ is larger for tractrix and the curvature is negative.
Both of these constants (curvature and $C_M$) contribute to the lower bound in Theorem \ref{thm:dist}.
Similarly, we show results of number of linear regions divided by the number of neurons upon changing architectures, consequently the number of neurons, for the two manifolds in Figure \ref{fig:multiarch}, averaged over 30 runs.
Note that this experiment observes the effect of $C_M \times C_{\text{grad}}$, since changing the architecture also changes $C_{\text{grad}}$ and the variation in $C_{\text{grad}}$ is quite low in magnitude as observed empirically by \citet{Hanin2019ComplexityOL}.
The empirical observations are consistent with our theoretical results.
We observe that the number of linear regions starts off close to $\# \text{neurons}$ and remains close throughout the training process for both the manifolds.
This supports our theoretical results (Theorem 3.3) that the constant $C_M$, which is distinct across the two manifolds, affects the number of linear regions throughout training.
The tractrix has a higher value of $C_M$ and that is reflected in both Figures \ref{fig:graph_lr} and \ref{fig:graph_dists}.
Note that its relationship is inverse to the average distance to the boundary region, as per Theorem \ref{thm:dist}, and it is reflected as training progresses in Figure \ref{fig:graph_dists}.
This is due to different ``shrinking'' of vectors upon being projected to the tangent space (Section \ref{sec:intuit}).

\begin{figure*}
\begin{multicols}{2}
    \centering
    \includegraphics[width=.4\textwidth]{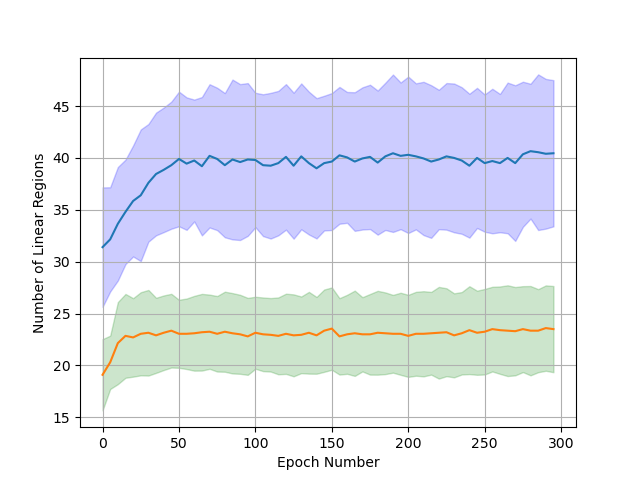}
    \caption{Graph of number of linear regions for tractrix (blue) and sphere (orange). The shaded regions represent one standard deviation. Note that the number of neurons is 26 and the number of linear regions are comparable to 26 but different for both the manifolds throughout training.}
    \label{fig:graph_lr}
    \includegraphics[width=.4\textwidth]{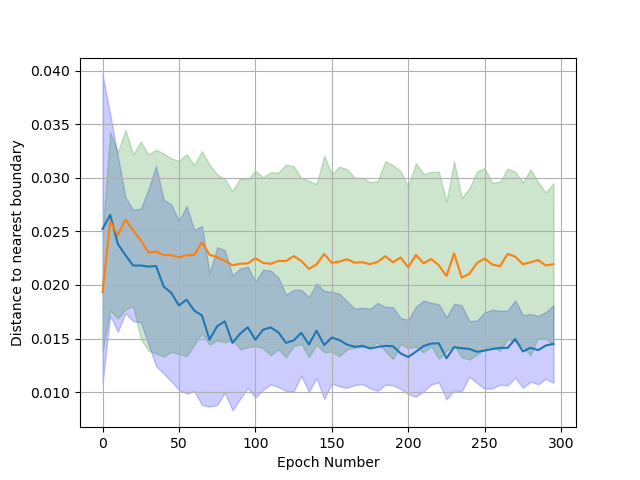}
    \caption{Graph of distance to linear regions for tractrix (blue) and sphere (orange). The distances are normalized by the maximum distance on the range, for both tractrix and sphere. The shaded regions represent one standard deviation. }
    \label{fig:graph_dists}
\end{multicols}

\begin{multicols}{2}
    \centering
    \includegraphics[width=.40\textwidth]{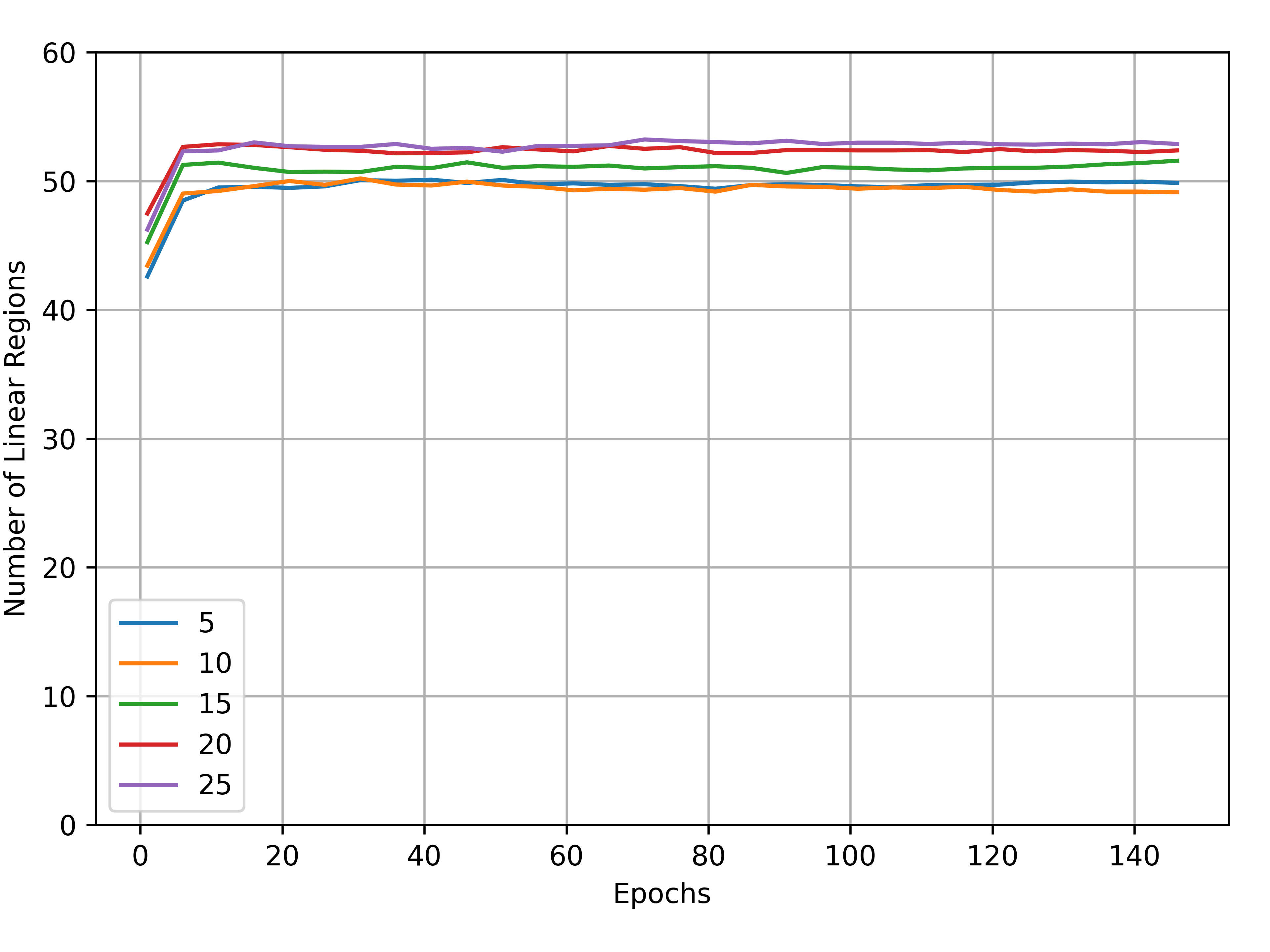}
    \caption{We observe that as the dimension $\nin$ is increased, while keeping the manifold dimension constant, the number of linear regions remains proportional to number of neurons (26).}
    \label{fig:nin_dims}
    \includegraphics[width=.40\textwidth]{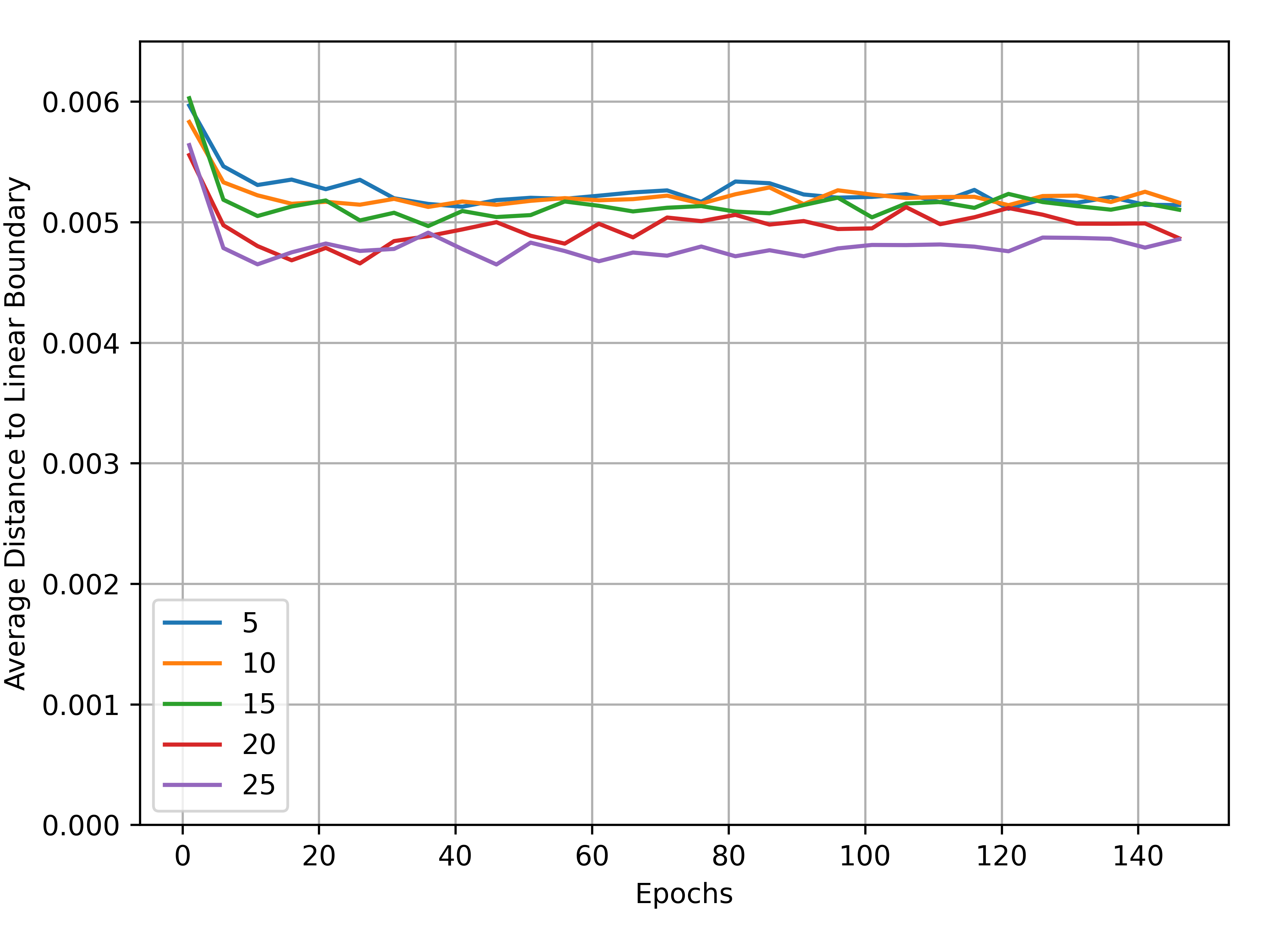}
    \caption{We observe that as the dimension $\nin$ is increased, while keeping the manifold dimension constant, the average distance varies very little.}
    \label{fig:nin_dists}
\end{multicols}

\begin{multicols}{2}
    \centering
    \includegraphics[width=.40\textwidth]{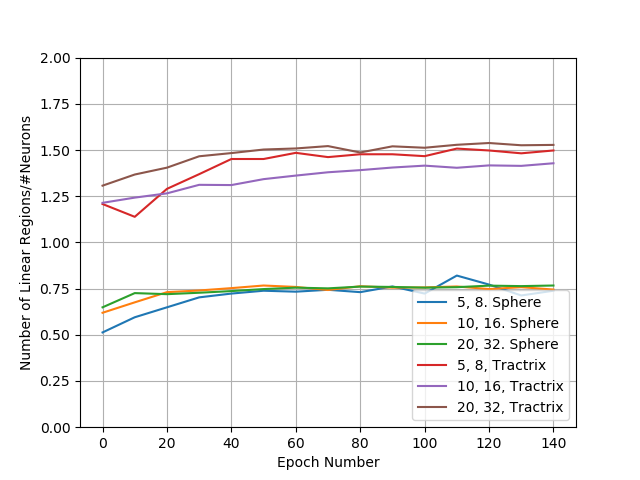}
    \caption{The effects of changing the architecture on the number of linear regions. We observe that the value of $C_M$ effects the number of linear regions proportionally. The number of hidden units for three layer networks are in the legend along with the data manifold.}
    \label{fig:multiarch}
    \includegraphics[width=.40\textwidth]{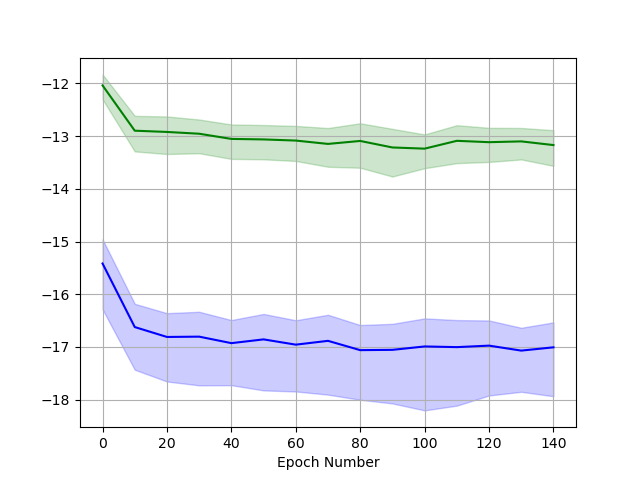}
    \caption{We observe that the log density of number of linear regions is lower on the manifold (blue) as compared to off the manifold (green). This is for the MetFaces dataset.}
    \label{fig:overfit}
\end{multicols}
\end{figure*}

\subsection{Varying Input Dimensions} \label{sec:varyn}
%\paragraph{Varying $\nin$} \label{sec:varyn}
To empirically corroborate the results of Theorems 2 and 3 we vary the dimension $\nin$ while keeping $m$ constant.
We achieve this by counting the number of linear regions and the average distance to boundary region on the 1D circle as we vary the input dimension in steps of 5.
We draw samples of 1D circles in $\mathbb R^{\nin}$ by randomly choosing two perpendicular basis vectors.
We then train a  network with the same architecture as  the previous section on the periodic target function ($a \sin (\nu \theta)$) as defined above.
The results in Figure \ref{fig:nin_dims} shows that the quantities stay proportional to $\# neurons$, and do not vary as $\nin$ is increased, as predicted by our theoretical results.
Our empirical study asserts how the relevant upper and lower bounds, for the setting where data lies on a low-dimensional manifold, does not grow exponentially with $\nin$ for the density of linear regions in a compact set of $\mathbb R^{\nin}$ but instead depend on the intrinsic dimension.
Further details are in Appendix \ref{app:toyprob}.

\subsection{MetFaces: High Dimensional Dataset}
%\paragraph{MetFaces: High Dimensional Dataset}
\label{sec:metfaces}

Our goal with this experiment is to study how the density of linear regions varies across a low dimensional manifold and the input space.
To discover latent low dimensional underlying structure of data we employ a GAN.
Adversarial training of GANs can be effectively applied to learn a mapping from a low dimensional latent space to high dimensional data \citep{Goodfellow2014GenerativeAN}.
The generator is a neural network that maps $g: \mathbb R^k \to \mathbb R^{\nin}$.
We train a deep ReLU network on the MetFaces dataset with random labels (chosen from $0, 1$) with cross entropy loss.
As noted by \citet{Zhang2017UnderstandingDL}, training with random labels can lead to the DNN memorizing the entire dataset.

We compare the log density of number of linear regions on a curve on the manifold with a straight line off the manifold.
We generate these curves using the data sampled by the StyleGAN by  \citep{Karras2020TrainingGA}.
Specifically, for each curve we sample a random pair of latent vectors: $z_1, z_{2} \in \mathbb R^k$, this gives us the start and end point of the curve using the generator $g(z_1)$ and $g(z_2)$.
We then generate 100 images to approximate a curve connecting the two images on the image manifold in a piece-wise manner.
We do so by taking 100 points on the line connecting $z_1$ and $z_2$ in the latent space that are evenly spaced and generate an image from each one of them.
Therefore, the $i^{\text{th}}$ image is generated as: $z_i' = g(((100 - i) \times z_1 + i \times z_2)/100)$, using the StyleGAN generator $g$.
We qualitatively verify the images to ensure that they lie on the manifold of images of faces.
The straight line, with two fixed points $g(z_1)$ and $g(z_2)$, is defined as $x(t) = (1 - t) g(z_1) + t g(z_2)$ with $t \in [0, 1]$.
The approximated curve on the manifold is defined as $x'(t) = (1- t) g(z_i') + t g(z_{i + 1}')$ where $i = \texttt{floor}(100t)$.
We then apply the method from Section \ref{subsec:1dmethod} to obtain the number of linear regions on these curves.

The results are presented in Figure \ref{fig:overfit}.
This leads us to the key observation:
the density of linear regions is significantly lower on the data manifold and devising methods to ``concentrate'' these linear regions on the manifold is a promising research direction.
That could lead to increased expressivity for the same number of parameters.
We provide further experimental details in Appendix \ref{app:hdd}.

\section{Discussion and Conclusions}
There is significant  work in both supervised and unsupervised learning settings for non-Euclidean data \citep{Bronstein2017GeometricDL}.
Despite these empirical results most theoretical analysis is agnostic to data geometry, with a few prominent exceptions \citep{Cloninger2020ReLUNA, Shaham2015ProvableAP, SchmidtHieber2019DeepRN}.
We incorporate the idea of data geometry into measuring the effective approximation capacity of DNNs, deriving average bounds on the density of boundary regions and distance from the boundary when the data is sampled from a low dimensional manifold.
Our experimental results corroborate our theoretical results.
We also present insights into expressivity of DNNs on low dimensional manfiolds for the case of high dimensional datasets.
Estimating the geometry, dimensionality and curvature, of these image manifolds accurately is a problem that remains largely unsolved \citep{Brehmer2020FlowsFS, PerraulJoncas2013NonlinearDR}, which limits our inferences on high dimensional dataset to observations that guide future research.
We note that proving a lower bound on the number of linear regions, as done by \citet{Hanin2019ComplexityOL}, for the manifold setting remains open.
Our work opens up avenues for further research that combines model geometry and data geometry and can lead to empirical research geared towards developing DNN architectures for high dimensional datasets that lie on a low dimensional manifold.

\section{Acknowledgements}

This work was funded by L2M (DARPA  Lifelong Learning Machines program under grant number FA8750-18-2-0117), the Penn MURI (ONR under the PERISCOPE MURI Contract N00014- 17-1-2699), and the ONR Swarm (the ONR under grant number N00014-21-1-2200).
This research was conducted using computational resources and services at the Center for Computation and Visualization, Brown University.

We would like to thank Sam Lobel, Rafael Rodriguez Sanchez, and Akhil Bagaria for refining our work, multiple technical discussions, and their helpful feedback on the implementation details.
We also thank Tejas Kotwal for assistance on deriving the mathematical details related to the 1D Tractrix and sources for various citations.
We thank Professor Pedro Lopes de Almeida, Nihal Nayak, Cameron Allen and Aarushi Kalra for their valuable comments on writing and presentation of our work.
We thank all the members of the Brown robotics lab for their guidance and support at various stages of our work.
Finally, we are indebted to, and graciously thank, the numerous anonymous reviewers for their time and labor as their valuable feedback and thoughtful engagement have shaped and vastly refine our work.

\bibliographystyle{plainnat}
\bibliography{example_paper}

%%%%%%%%%%%%%%%%%%%%%%%%%%%%%%%%%%%%%%%%%%%%%%%%%%%%%%%%%%%%
\section*{Checklist}

%%% BEGIN INSTRUCTIONS %%%

\begin{enumerate}

\item For all authors...
\begin{enumerate}
  \item Do the main claims made in the abstract and introduction accurately reflect the paper's contributions and scope?
    \answerYes{}
  \item Did you describe the limitations of your work?
    \answerYes{}
  \item Did you discuss any potential negative societal impacts of your work?
    \answerNA{Our work is primarily theoretical with few toy experiments we do not see its applicability}
  \item Have you read the ethics review guidelines and ensured that your paper conforms to them?
    \answerYes{}
\end{enumerate}

\item If you are including theoretical results...
\begin{enumerate}
  \item Did you state the full set of assumptions of all theoretical results?
    \answerYes{See Appendix A for a list}
        \item Did you include complete proofs of all theoretical results?
    \answerYes{}
\end{enumerate}

\item If you ran experiments...
\begin{enumerate}
  \item Did you include the code, data, and instructions needed to reproduce the main experimental results (either in the supplemental material or as a URL)?
    \answerYes{See Appendix J}
  \item Did you specify all the training details (e.g., data splits, hyperparameters, how they were chosen)?
    \answerYes{See experimental sections in the Appendix and main body}
        \item Did you report error bars (e.g., with respect to the random seed after running experiments multiple times)?
    \answerYes{Except for the cases where there are multiple graphs that are overlapping (Figure 6,7, 8) because it would make interpreting them difficult.}
        \item Did you include the total amount of compute and the type of resources used (e.g., type of GPUs, internal cluster, or cloud provider)?
    \answerYes{Appendix J}
\end{enumerate}

\item If you are using existing assets (e.g., code, data, models) or curating/releasing new assets...
\begin{enumerate}
  \item If your work uses existing assets, did you cite the creators?
    \answerYes{}
  \item Did you mention the license of the assets?
    \answerYes{}
  \item Did you include any new assets either in the supplemental material or as a URL?
    \answerNo{}
  \item Did you discuss whether and how consent was obtained from people whose data you're using/curating?
    \answerNA{}
  \item Did you discuss whether the data you are using/curating contains personally identifiable information or offensive content?
    \answerNA{}
\end{enumerate}

\item If you used crowdsourcing or conducted research with human subjects...
\begin{enumerate}
  \item Did you include the full text of instructions given to participants and screenshots, if applicable?
    \answerNA{}
  \item Did you describe any potential participant risks, with links to Institutional Review Board (IRB) approvals, if applicable?
    \answerNA{}
  \item Did you include the estimated hourly wage paid to participants and the total amount spent on participant compensation?
    \answerNA{}
\end{enumerate}

\end{enumerate}

\newpage
%%%%%%%%%%%%%%%%%%%%%%%%%%%%%%%%%%%%%%%%%%%%%%%%%%%%%%%%%%%%

\appendix

\section{Assumptions}
\label{sec:assumptions}

We first make explicit the assumptions on the distribution of weights and biases.

\begin{enumerate}
    \item[\textbf{A1:}] The conditional distribution of any set of biases $b_{z_1}, ..., b_{z_k}$ given all other weights and biases has a density $\rho_{z_1, ..., z_k}(b_1, ..., b_k)$ with respect to Lebesgue measure on $\mathbb R^k$.
    
    \item[\textbf{A2:}] The joint distribution of all weights has a density with respect to Lebesgue measure on $\mathbb R^{\#\text{weights}}$.
    
    \item[\textbf{A3:}] The data manifold $M$ is smooth.
    
    \item[\textbf{A4:}] (Only needed for Theorem 3) the diameter of $M$ defined by $d_M =  \sup_{x, y \in M} \distfun_M(x, y)$ is finite.
    
    \item[\textbf{A5:}] (Only needed for Theorem 3) a geodesic ball in manifold $M$ has polynomial volume growth of order $m$.
\end{enumerate}

\section{Additional Background on Manifolds} \label{app:background}

We provide further background on the theory of manifolds.
In this section we first provide the background, definition and an interpretation for the \textbf{scalar curvature} of a manifold at a point.
Every smooth manifold is also equipped with a \textit{Riemannian metric tensor} (or metric tensor in short).
Given any two vectors, $v$ and $w$, in the tangent space of a point $x$ on a manifold $M$, the metric tensor defines a parallel to the dot product in Euclidean spaces.
The metric tensor, at a point $x$, is defined by the smooth functions $g_{ij}: M \to \mathbb R, i, j \in \{1, ..., k\}$.
Where the matrix defined by
\begin{equation*}
    G_x = [g_{ij}(x)] = \begin{bmatrix}
    g_{11}(x) & \hdots & g_{1n}(x) \\
    \vdots & \ddots & \vdots \\
    g_{n1}(x) & \hdots & g_{nn}(x)
    \end{bmatrix}
\end{equation*}
is symmetric and invertible.
The inner product of $u, v \in T_x M$ is then defined by $\langle u, v\rangle_M = u^T G_x v$.
the inner product is symmetric, non-degenerate, and bilinear, i.e.
\begin{align*}
    \langle ku, v \rangle_M =& k \langle u, v \rangle_M =  \langle u, kv \rangle_M, \\
    \langle u + w, v \rangle_M =& \langle u , v \rangle_M + \langle w, v \rangle_M, \\
    \langle u, v \rangle_M =& \langle v, u \rangle_M.
\end{align*}
As can be seen, these properties also hold for the Euclidean inner product (with $G_x = I$ for all $x$).
Let the inverse of $G = [g_{ij}(x)]$ be denoted by $[g^{ij}(x)]$.
Building on this definition of the metric tensor the Ricci curvature tensor is defined as
\begin{align*}
    R_{ij} = & -\frac{1}{2} \sum_{a, b = 1}^n \Big ( \frac{\partial^2 g_{ij}}{\partial x_a \partial x_b} + \frac{\partial^2 g_{ab}}{\partial x_i \partial x_j} - \frac{\partial^2 g_{ib}}{\partial x_j \partial x_a} - \frac{\partial^2 g_{jb}}{\partial x_i \partial x_a}\Big ) g^{ab}\\
    & + \sum_{a, b, c, d = 1}^n \Big ( \frac{1}{2} \frac{\partial g_{ac}}{\partial x_i} \frac{\partial g_{bd}}{\partial x_j } + \frac{\partial g_{ic}}{ \partial x_a} \frac{\partial g_{jd}}{\partial x_b} -\frac{\partial g_{ic}}{ \partial x_a} \frac{\partial g_{jb}}{\partial x_d} \Big ) g^{ab} g^{cd} \\
    & - \frac{1}{4} \sum_{a,b,c,d = 1}^n \Big ( \frac{\partial g_{jc}}{\partial x_i} + \frac{\partial g_{ic}}{\partial x_j} - \frac{\partial g_{ij}}{\partial x_c} \Big ) g^{ab} g^{cd}.
\end{align*}

For geometric interpretations of the above tensors we refer the reader to the work by \citet{Loveridge2004PhysicalAG}.

Another quantity, from the theory of manifolds, which we utilise in our proofs and theorems, is scalar curvature (or Ricci curvature).
The curvature is a measure how much the volume of a geodisic ball on the manifold M, e.g. $S^2$, deviates from a $d - 1$ sphere in the flat space, e.g. $\mathbb R^3$.
The volume on the manifold deviates by an amount proportional to the curvature.
We illustrate this idea in figure \ref{fig:circlecomp}.
We refer the reader to works by \citet{Gray1974TheVO} and \citet{Wan2016GEOMETRICIO} for further technical details.
Since our main theorems relate to the volume of linear regions the scalar curvature plays an important role.
Formally, the scalar curvature of a manifold $M$ at a point $x$ with metric tensor $[g_{ij}]$ and Ricci tensor $[R_{ij}]$ is defined as
\begin{equation*}
    C = \sum_{i, j = 1}^n g^{ij} R_{ij}.
\end{equation*}

\begin{figure}
    \centering
    \subfigure[]{\includegraphics[width=0.42\textwidth]{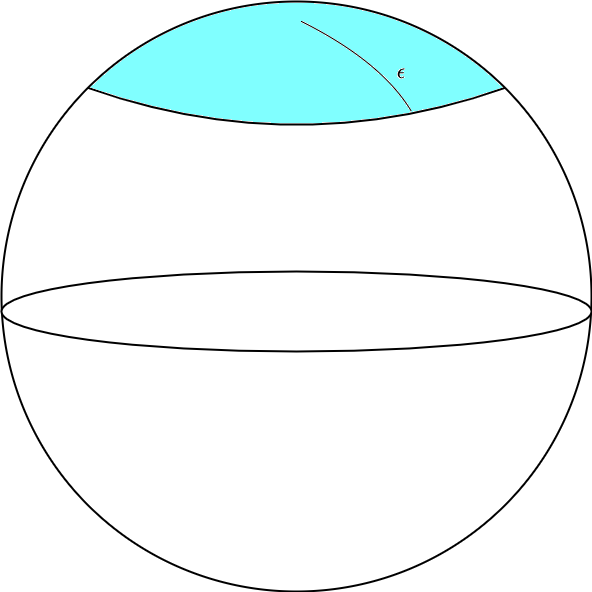}}
    \hspace{2em}
    \subfigure[]{\includegraphics[width=0.42\textwidth]{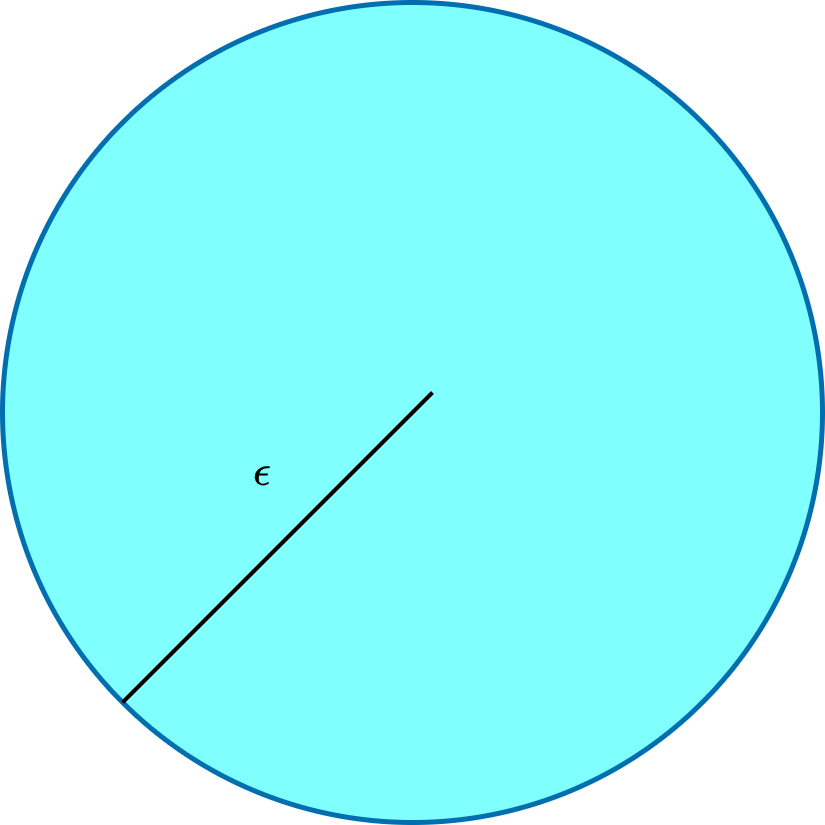}}
    \caption{The geodesic circle on $S^2$ (blue region in (a)) does not have the same area as the flat circle (b), both of radius $\epsilon$. One can imagine cutting the blue top off the sphere's surface and trying to ``flatten'' it. Such an effort will lead to failure, if the material of the sphere does not "stretch", since the geodesic ball, on $S^2$, cannot be mapped to a circle in $\mathbb R^2$ in a distance preserving manner. Thus, the area of the two blue regions in (a) and (b) vary. This deviation in the area spanned by the two spheres, despite their radii being the same, is proportional to the scalar curvature. }
    \label{fig:circlecomp}
\end{figure}

Another important concept is that of \textbf{Hausdorff measure}.
Since the volumes are ``distorted'' on a manifold it requires careful consideration when defining a measure and integrating using it on a manifold.
The $m-$dimensional Hausdorff measure, of a set $S$, is defined as
\begin{equation*}
    H^m(S) \defeq \sup_{\delta > 0}  \inf \Big \{ \sum_{i = 1}^{\infty} (\text{diam } U_i)^d | S \subseteq\cup_{i = 1}^{\infty} U_i, \text{diam } U_i < \delta  \Big \}.
\end{equation*}

Next we introduce the definition of the \textbf{differential map} that is used in Definition 3.1, for the determinant of the Jacobian.
The differential map of a smooth function $H$ from a manifold $M$ to a manifold $S$ at a point $x \in M$ is the smooth map $ dH: T_x M \to T_x S$ such that the tangent vector corresponding to any smooth curve $\gamma: I \to M$ at $x$, $\gamma'(0) \in T_x M$, maps to the tangent vector of $H \circ \gamma$ in $T_{H(x)} N$.
This is the analog of the total derivative of ``vanilla calculus''.
More intuitively, the differential map captures how the function changes along different directions on $N$ as its input changes along different directions on $M$, this also has an analog to how rows of the Jacobian matrix are viewed in calculus.
In Definition 3.1 we use the specific case where the function $H$ maps from manifold $M$ to the Euclidean space $\mathbb R^k$ and the tangent space of a Euclidean space is the Euclidean space itself.
Finally, a paralellepiped's, $P$ in $T_x M$, mapping via the differential map gives us the points in $\mathbb R^k$ that correspond to this set $P$.

\section{Related Work}
\label{app:relwork}

There have been various approaches to explain the efficacy of DNNs in approximating arbitrarily complex functions.
We briefly touch upon two such promising approaches.
Broadly, the theory of DNNs can be viewed from two lenses: expressive power \citep{Hornik1989MultilayerFN, Bartlett1998AlmostLV, Poole2016ExponentialEI, Raghu2017OnTE, Kawaguchi2017GeneralizationID, Neyshabur2018APA, Hanin2019UniversalFA} and learning dynamics \citep{Saxe2014ExactST, Su2016ADE, Smith2018ABP, Jacot2018NeuralTK, Lee2019WideNN, Arora2019ACA, Arora2019ImplicitRI}.
These approaches are not independent of one another but complementary.
For example, \citet{Kawaguchi2017GeneralizationID} argue theoretically how the family of DNNs generalize well despite the large capacity of the function class.
\citet{Neyshabur2018APA} provide PAC-Bayes generalization bounds which are improved upon by \citet{Arora2018StrongerGB}.
\citet{Hanin2019UniversalFA} shows that Deep ReLU networks of finite width can approximate any continuous, convex or smooth functions on a unit cube.
These works look at DNNs from the lens of expressive power.
More recently, there has been a surge in explaining how various algorithms arrive at these almost accurate function approximations by applying different theoretical models of DNNs.
\citet{Jacot2018NeuralTK} provide results for convergence and generalization of DNNs in the infinite width limit by introducing a the neural tangent kernel (NTK).
\citet{Hanin2020FiniteDA} provide finite depth and width corrections for the NTK.
Another line of work within the learning dynamics literature looks at implicit regularization that emerge from the learning algorithm and over-parametrised DNNs \citep{Arora2019ACA, Arora2019ImplicitRI, Du2018AlgorithmicRI, Liang2019FisherRaoMG}.

Researchers have begun to incorporate data geometry into the theoretical analyses of DNNs by applying the assumption that the data lies on a general manifold.
First we note the works looking at DNNs from the lens of expressive power combined with the idea of data geometry.
\citet{Shaham2015ProvableAP} demonstrate that the size of the neural network depends on the curvature of the data manifold and the complexity of the function, whilst depending weakly on the input data dimension, for their construction of sparsely-connected 4-layer neural networks.
\citet{Cloninger2020ReLUNA} show that their construction of deep ReLU nets achieve near optimal approximation rates which depend only on the intrinsic dimensionality of the data.
\citet{Chen2019EfficientAO} exploit the low dimensional structure of data to enhance the function approximation capacity of Deep ReLU networks by means of theoretical guarantees.
\citet{SchmidtHieber2019DeepRN} shows that sparsely connected deep ReLU networks can approximate a Holder function on a low dimensional manifold embedded in a high dimensional space.
Simultaneously, researchers have incorporated data geometry into the learning dynamics line of work \citep{Goldt2020ModellingTI, Paccolat2020GeometricCO, Buchanan2021DeepNA, Wang2021DeepNP}.
\citet{Buchanan2021DeepNA} apply the NTK model to study how DNNs can separate two curves, representing the data manifolds of two separate classes, on the unit sphere.
\citet{Goldt2020ModellingTI} introduce the Hidden Manifold Model for structured data sets to capture the dynamics of two-layer neural networks trained with stochastic gradient descent.
\citet{Rahaman2019OnTS} provide empirical results on which data manifolds are learned faster.
Finally, the work by \citet{novak2018sensitivity} comes the closes in studying the number of linear regions on the data manifold.
They study the change in input output Jacobian, and as a consequence the number of linear regions, for DNNs with piece-wise linearities.
They provide empirical studies by counting the number of linear regions along lines connecting data points as a proxy for number of linear regions on the data manifold. 

Our work fits into the study of expressive power of DNNs.
The number of linear regions is a good proxy for the \textit{practical} expressive power or approximation capacity of Deep ReLU networks \citep{Montfar2014OnTN}.
The results surrounding the density of linear regions make the fewest simplifying assumptions both on the data and the architecture of the DNN.
The results by \citet{Hanin2019ComplexityOL} bound the number of linear regions orders of magnitude tighter than previous results by deriving bounds for the average case and not the worst case.
Moreover, they demonstrate the validity empirically in a setting with very few simplifying assumptions.
We introduce the manifold hypothesis to this setting in order to obtain tighter bounds for the first time.
This introduces a toolbox of ideas from differential geometry to analyse the approximation capacity of deep ReLU networks.

In addition to the theoretical works listed above, there has been significant empirical work that applies DNNs to non-Euclidean data \citep{Bronstein2017GeometricDL, Bronstein2021GeometricDL}.
Here the data is assumed to be sampled from manifolds with certain geometric properties.
For example, \citet{Ganea2018HyperbolicNN} design DNNs for data sampled from Hyperbolic spaces of arbitrary dimensionality and modify the forward and backward passes accordingly.
There have been numerous applications of modified DNNs, namely graph convolutional networks, to graph data that incorporate the idea that graphs are discrete samples from a smooth manifold \citep{Henaff2015DeepCN, Monti2017GeometricDL, Kipf2017SemiSupervisedCW}, see the survey by \citet{Wu2019ACS} for a comprehensive review.
Graph convolutional networks have also been applied to point cloud data for applications in graphics \citep{Qi2017PointNetDL, Wang2019DynamicGC}.

\section{Proof Sketch} \label{app:proofsketch}

In this section we provide an overview of how the three main theorems are proved.
Theorem 3.2 provides an equality for measuring the volume of $m - k$ dimensional boundary regions on the manifold.
To this effect, we introduce the idea of viewing boundary regions as submanifolds on the data manifold instead of hyperplanes (Proposition 6).
We then prove an equality between the volume of boundary regions and the Jacobian of the neurons over the manifold.
We utilise the smooth coarea formula that, intuitively, is applied to integrate a function using level sets on a manifold.
This completes the proof for Theorem 3.2.

To prove Theorem 3.3 we first prove that the Jacobian of a function on a manifold can be denoted using the volume of paralellepiped of vectors in the ambient space subject to a linear transform (Proposition 8).
Using this result and combining it with Theorem 3.2 we can then give an inequality for the density of linear regions.
As can be expected this volume depends on the aforementioned projection, which in turn is related to the geometry of the manifold.

Finally, for proving Theorem 3.4 we first provide an inequality over the tubular neighbourhood of the boundary region.
We then use this result to lower bound the geodesic distance between the boundary region and any random point on the manifold.
The proof strategy follows that of \citet{Hanin2019ComplexityOL} but there are major deviations when it comes to accounting for the geometry of the data manifold.
To the best of our knowledge, we are utilising elements of differential topology that are unique to machine learning when it comes to developing a theoretical understanding of DNNs.

\section{Proof of Theorem 3.2}
\label{sec:jacobproof}

We follow the proof strategy used by \citet{Hanin2019ComplexityOL} but deviate from it to account for our setting where $x \in M$.
Let $S_z$ be the set of values at which the neuron $z$ has a discontinuity in the differential of its output (or the neuron switches between the two linear regions of the piece-wise linear activation $\sigma$),
\begin{equation*}
    S_z \defeq \{ x \in \mathbb R^{n_{\text{in}}} | z(x) - b_z = 0 \}.
\end{equation*}
We also have
\begin{align*}
    \mathcal O \defeq \Big \{ x \in \mathbb R^{n_{\text{in}}} | & \forall j = 1,...,L \ \exists \text{ neuron } z \text{ with } l(z) = j \text{ s.t. } \sigma'(z(x) - b_z) \neq 0\Big \}.
\end{align*}
Further,
\begin{equation*}
    \widetilde{S_z} \defeq S_z \cap \mathcal O.
\end{equation*}
We state propositions 9 and 10 by \citet{Hanin2019ComplexityOL} as we apply them to prove Theorem 3.2, relabeling them as needed.
\begin{proposition}
\label{prop:bass}
\textbf{(Proposition 9 by \citet{Hanin2019ComplexityOL})} Under assumptions A1 and A2, we have, with probability 1,
\begin{equation*}
    B_F = \bigcup_{\text{neurons z}} \widetilde{S_z}.
\end{equation*}
\end{proposition}

By extending the notion of $S_z$ to multiple neurons we have
\begin{equation*}
    \widetilde{S}_{z_1,...,z_k} \defeq \bigcap_{j = 1}^{k} \widetilde{S}_{z_j}, 
\end{equation*}
meaning that the set $\widetilde{S}_{z_1,...,z_k}$ is, intuitively, the collection of inputs in $\mathbb R^{\text{in}}$ where the neurons $z_j, j = 1,...,k,$ switch between linear regions for $\sigma$ and at which the output of $F$ is affected by the outputs of these neurons.
We refer the reader to section B of the appendix in the work by \citet{Hanin2019ComplexityOL} for an intuitive explanation of proposition \ref{prop:bass}. 
Before proceeding we provide a formal definition and intuition for the set $\mathcal B_{F, k}$,
\begin{align*}B_{F, k} = \{ &x | x \in B_F \setminus \{\mathcal B_{F, 0} \cup ... \cup \mathcal B_{F, k - 1} \} = \mathcal B_{F, -k} \text{ and for any ball of radius } \epsilon > 0, \\
&B(x, \epsilon) \cap \mathcal B_{F, -k} \text{ is subset to a } n-k \text{ dimensional hyperplane} \}.
\end{align*} 
Following the explanation provided by \citet{Hanin2019ComplexityOL}, $\mathcal B_{F, k}$ is the $\nin - k$ dimensional piece of $\mathcal B_F$.
Suppose the boundaries of linear regions for $\nin = 2$ are unions of polygon boundaries, as depicted in Figure 2 of the main body of the paper, then $\mathcal B_{F, 1}$ are all the open line segments of these polygons and $\mathcal B_{F, 2}$ are the end points.
Next we state Proposition 10 by \citet{Hanin2019ComplexityOL}.
\begin{proposition}
\label{prop:dims}
\textbf{(Prosposition 10 by \citet{Hanin2019ComplexityOL})} Fix $k = 1, ..., n_{\text{in}}$, and $k$ distinct neurons $z_1, ..., z_k$ in $F$. Then, with probability 1, for every $x \in B_{F, k}$ there exists a neighbourhood in which $B_{F, k}$ coincides with a $n_{\text{in} - k}-$dimensional hyperplane.
\end{proposition}

We now present Proposition \ref{prop:manifolddims}, and its proof, which incorporates the additional constraint that $x \in M$, which is an $m$-dimensional manifold in $\mathbb R^{\nin}$.  To prove the proposition we need the definition of tranversal intersection of two manifolds \citep{guilleminpollock1974}. 
\begin{definition} \label{def:trans}
\textit{
Two submanifolds, $M_1$ and $M_2$, of $S$ are said to intersect transversally if at every point of intersection their tangent spaces, at that point, together generate the tangent space of the manifold, $S$, by means of linear combinations. Formally, for all $x \in M_1 \cap M_2$
\begin{equation*}
    T_x S = T_x M_1 + T_x M_2,
\end{equation*}
if and only if $M_1$ and $M_2$ intersect transversally.}
\end{definition}

For example, given a 2D hyperplane, $P$, and the surface of a 3D sphere, $S^2$, intersect in the ambient space $\mathbb R^3$. We have that this intersection is transverse if and only if $P$ is not tangent to $S_2$. For the case where a 2D hyperplane, $\bar{P}$, intersects with $S^2$ at a point $p$ but does not intersect tranversally it coincides exactly with the tangent plane of $S^2$ at point $\{p\} = S ^2 \cap P$, i.e. $T_p S = P$. Note that in either case the tangent space of the 2D hyperplane $P$ at any point of intersection is the plane itself.

\begin{proposition}
\label{prop:manifolddims}
Fix $k = 1, ..., m$ and $k$ distinct neurons $z_1, ..., z_k$ in $F$. Then, with probability 1, for every $x \in B_{F, k} \cap M$ there exists a neighbourhood in which $B_{F, k}$ coincides with an $m - k$ dimensional submanifold in $\mathbb R^{\text{in}}$.
\end{proposition}
\begin{proof}
From Proposition \ref{prop:dims} we already know that $B_{F, k}$ is a $n_{\text{in}} - k$-dimensional hyperplane in some neighbourhood of $x$, with probability 1, for any $x \in B_{F, k} \cap M$. Let this hyperplane be denoted by $P_k$. This is an $n-k$ dimensional submanifold of $\mathbb R^{\nin}$. The tangent space of this hyperplane at $x$ is the hyperplane itself. Therefore, from assumptions A1 and A2 we have that the probability that this hyperplane intersects the manifold $M$ transversally with probability 1.
In other words the probability that this plane $P_k$ contains or is contained in $T_x M$ is $0$.
Finally, we have the intersection, $M \cap H_k$, has dimension $\dim(M) + \dim(H_k) - \nin$ \citep{guilleminpollock1974}, which is equal to $m - k$.
\end{proof}
One implication of Proposition \ref{prop:manifolddims} is that for any $k \leq m$ the $m - (k + 1)$ dimensional volume of $ B_{F, k} \cap M$ is 0. In addition to that, Proposition \ref{prop:manifolddims} implies that, with probability 1,
\begin{equation}\label{eq:summation}
    \vol_{m - k}(\calB_{F, k}) = \sum_{\text{distinct neurons }z_1, ..., z_k} \vol_{m - k}( \widetilde{S}_{z_1,...,z_k} \cap M).
\end{equation}

The final step in the proof of Theorem 3.2 is to prove the following result.
\begin{proposition}\label{prop:jacobsum}
Let $z_1, ..., z_k$ be distinct neurons in $F$ and $k \leq m$. Then for a bounded $m-$Hausdorff measurable manifold $M$ embedded in $\mathbb R^{\nin}$,
\begin{align*}
    \mathbb E\Big[ &\vol_{m - k}\Big ( \widetilde{S}_{z_1, ..., z_k} \cap M\Big ) \Big]  = \int_M \mathbb E \Big [ Y_{z_{1}, ..., z_{k}}(x) \Big ] dx,
\end{align*}
where $Y_{z_{1}, ..., z_{k}}(x)$ equals
\begin{equation*}
    J_{m, H_k}^M(x) \rho_{b_{1}, ..., b_{k}}(z_{1}(x), ..., z_{k}(x)),
\end{equation*}
times the indicator function of the event that $z_{j}$, for $j = 1, ..., k$, is good at $x$ for every $j$ and $H_{k}: \mathbb R^{\nin} \to \mathbb R^k$ is such that $H_{k}(x) = [z_{1}(x), ..., z_{k}(x)]^T$. The expectation is over the distribution of weights and biases.
\end{proposition}
\begin{proof}
Let $z_1, ..., z_k$ be distinct neurons in $F$ and $M$ be an $m-$dimensional compact Haudorff measurable manifold. We seek to compute the mean of $\vol_{m - k}(\widetilde{S}_{z_1, ..., z_k} \cap M)$ over the distribution of weights and biases. We can rewrite this expression as 
\begin{equation}\label{eq:breakdown}
    \int_{S_{z_1, ..., z_k} \cap M} \textbf{1}_{z_j \text{ is good at } x} d \vol_{m - k}(x). 
\end{equation}
The map $H_k$ is Lipschitz and $C^1$ almost everywhere. We first note the smooth coarea formula (theorem 5.3.9 by \citet{Krantz2008GeometricIT}) in context of our notation. Suppose $m \geq k$ and $H_k: \mathbb R^{\nin} \to \mathbb R^k$ is $C^1$ and $M \subseteq \mathbb R^{\nin}$ is an $m-$dimensional $C^1$ manifold in $\mathbb R^{\nin}$, then
\begin{align}\label{eq:smoothcoarea}
    \int_M g(x)& J^{M}_{k, H_k}(x) d \vol_{m}(x) = \int_{\mathbb R^k} \int_{M \cap H_k^{-1}(y)} g(y) d \vol_{m -k}(y) d \vol_{k}(x),
\end{align}
for every $\mathcal H^m$-measurable function $g$ where $J^{M}_{k, H_k}$ is as defined in Definition 3.1.

We denote preactivations and biases of neurons as $ \textbf{z}(x) = [z_1(x), ..., z_k(x)]^T $ and $ \textbf{b}_{\textbf{z}} = [b_{z_1}, ..., b_{z_k}]^T $. From the notation in A1, we have that
\begin{equation*}
    \rho_{\textbf{b}_{\textbf{z}}} = \rho_{b_{z_1}, ..., b_{z_k}},
\end{equation*}
is the joint conditional density of $b_{z_1}, ..., b_{z_k}$ given all other weights and biases. The mean of the term in equation \ref{eq:breakdown} over the conditional distribution of $b_{z_1}, ..., b_{z_k}$, $\rho_{\textbf{b}_{\textbf{z}}}$, is therefore
\begin{equation}\label{eq:breakdist}
    \int_{\mathbb R^k} \textbf{b} d \vol_k(\textbf{b}) \int_{\{\textbf{z} = \textbf{b}\} \cap M} \textbf{1}_{z_j \text{ is good at } x} d \vol_{m - k}(x),
\end{equation}
where we denote $[b_1, ..., b_k]^T$ as $\textbf{b}$. Thus applying the smooth co-area formula (Equation \ref{eq:smoothcoarea}) to the expression in \ref{eq:breakdist} shows that the average \ref{eq:breakdown} is equal to
\begin{equation*}
    \int_{M} Y_{z_{1}, ..., z_{k}}(x) dx.
\end{equation*}
Finally, we take the average over the remaining weights and biases and commute the expectation with the $dx$ integral. We can do this since the integrand is non-negative. This gives us the result:
\begin{equation}
    \mathbb E\Big[ \vol_{m - k}\Big ( \widetilde{S}_{z_1, ..., z_k} \cap M\Big ) \Big]  = \int_M \mathbb E \Big [ Y_{z_{1}, ..., z_{k}}(x) \Big ] dx,
\end{equation}
as required.
\end{proof}
Finally, taking the summation over all possible sets of distinct neurons $z_1, ... ,z_k$ and combining equation \ref{eq:summation} with Proposition \ref{prop:jacobsum} completes the proof for Theorem 3.2.

\section{Proof of Theorem 3.3} \label{app:proofthm2}
To prove the upper bound in Theorem 3.3 we first show that the (determinant of) Jacobian for the function $H_k:M \to \mathbb R^k$, $H_k(x) = [z_1(x), ..., z_k(x)]^T$ , as defined in 3.1 is equal to the volume of the parallelopiped defined by the vectors $\phi_{H_k}(\nabla z_j(x))$, for $j = 1,...,k$, where $\phi_{H_k}:\mathbb R^k \to T_x M$ is an orthogonal projection onto the orthogonal complement of the kernel of the differential $D_M H_k$. 
Intuitively, this shows that with the added assumption $x \in M$ in Theorem 3.3 how exactly we can incorporate the geometry of the data manifold $M$ into the upper bound provided by \citet{Hanin2019ComplexityOL} in corollary 7.

\begin{proposition}\label{prop:linearmap}
Given $H_k:M \to \mathbb R^k$ such that $H_k(x) = [z_1(x), ..., z_k(x)]^T$ and the differential $D_M H_k$ is surjective at $x$ then
\begin{equation}\label{prop:lineartransform}
    J_{k, H_k}^M(x) = \sqrt{\det(\text{Gram}(\phi_{H_k}(\nabla z_1(x)), ..., \phi_{H_k}(\nabla z_k(x))))},
\end{equation}
where $\phi_{H_k}: \mathbb R^n \to \mathbb R^k$ is a linear map and \textit{Gram} denotes the Gramian matrix.
\end{proposition}
\begin{proof}
We first define the orthogonal complement of the kernel of the differential $D_M H_k$.
For a manifold $M \subset \mathbb R^n$ and a fixed point $x$ we have that $T_x M$ is a $m-$dimensional hyperplane.
If we choose an orthonormal basis $e_1, ..., e_n$ of $\mathbb R^n$ such that $e_1, ..., e_m$ spans $T_x M$ for a fixed $x$ we can denote all vectors in $T_x M$ using $m$ coordinates corresponding to this basis.
Therefore, for any vector $y \in \mathbb R^k$ we can get the orthogonal projection of $y$ onto $T_x M$ using a $m \times n$ matrix which we denote as $P_x$, where $P_x y$ (matrix multiplied by a vector) represents a vector in $T_x M$ corresponding to the basis $e_1, ..., e_m$.
For any manifold $M$ in $R^n$ and function $H_k: M \to \mathbb R^k$ we have that $D_M H_k: T_x M \to \mathbb R^k$ at a fixed point $x$ is linear function. 
Therefore we can write $D_M H_k(v) = A v$ where $v \in T_x M$ is denoted using the aforementioned basis of $T_x M$. This implies that $A$ is a $k \times m$ matrix.
Therefore, the kernel of $D_M H_k$ for a fixed point $x \in M$ is
\begin{equation*}
    \ker(D_M H_k) = \Big \{ z | A z = 0 \text{ and } z \in T_x M \Big \}.
\end{equation*}
Since we can create a canonical basis for the space $\ker(D_M H_k)$ starting from the basis $e_1, ..., e_m$ in $R^n$ using the Gram-Schmidt process given the matrix $A$ we have that for any $y \in R^n$ we can project it orthogonally onto $\ker(D_M H_k)$.
The orthogonal complement of $\ker(D_M H_k)$ is therefore defined by
\begin{align*}
    \ker(D_M H_k)^{\perp} = \Big \{& a | a \cdot z = 0 \text{ for all } z \in \ker(D_M H_k) \text{ and } a \in T_x M \Big \}.
\end{align*}
Similar to the previous argument, we construct a canonical basis starting from $e_1, ..., e_m$ for $\ker(D_M H_k)^{\perp}$ and therefore we can denote the orthogonal projection onto $\ker(D_M H_k)^{\perp}$ as a linear transformation.
We denote this linear projection for fixed $x$ using $\phi_k$.

We denote the basis vectors $e_1, ...., e_m$ as a $m \times n$ matrix $ E$ where each row $i$ corresponds to the vector $e_i$.
Therefore, the orthogonal projection of any vector $y \in \mathbb R^{n}$ is $E y$. 
Now we can get the matrix $A$ using $E \nabla z_j(x)$ corresponding to each row $j$ for $j = 1, ..., m$.
This uses the fact that the direction of steepest ascent on $z_j(x)$ restricted to the tangent space $T_x M$ of the manifold $M$ is an orthogonal projection of the direction of steepest ascent in $\mathbb R^n$.

Finally, from lemma 5.3.5 by \citet{guilleminpollock1974} we have that
\begin{equation*}
    J_{k, H_k}^M(x) = \mathcal H^k(D_M H_k (P))/ \mathcal H^k(P),
\end{equation*}
for any parallelepiped $P$ contained in $(\ker(D_M H_k))^{\perp}$.
Arguing similar to the proof of lemma 5.3.5 by \citet{guilleminpollock1974} we get that
\begin{align*}
    J_{k, H_k}^M(x) =& \sqrt{\det((A)^T A)} = \sqrt{\det{\text{Gram}(E \nabla z_1 (x), ..., E \nabla z_k(x))}},
\end{align*}
thereby showing that $\phi_{H_k}(y) = Ey$ is a linear mapping.
\end{proof}

Although we state Proposition \ref{prop:linearmap} for neurons $z_j(x), j = 1, ...,k$  in the proof, it applies to any function that satisfy the conditions laid out in the proposition. Equipped with Proposition \ref{prop:linearmap} we prove Theorem 3.3.
When the weights and biases of $F$ are independent obtain an upper bound on $\rho_{b_{z_1}, ..., b_{z_k}}(b_1, ..., b_k)$ as
\begin{equation*}
    \Pi_{j=1}^{k} \rho_{b_{z_j}}(b_1, ..., b_k) \leq \Big ( \sup_{\text{neurons }z} \rho_{b_z}(b) \Big )^k = C_{\text{bias}}^k.
\end{equation*}

Hence,
\begin{equation*}
    Y_{z_1, ..., z_k} \leq C_{\text{bias}}^k J_{k, H_k}^M.
\end{equation*}
From Proposition \ref{prop:lineartransform} we have that $J_{k, H_k}^M$ is equal to the $k$-dimensional volume of the paralellopiped spanned by $\phi_x(\nabla z_j(x))$ for $j = 1, ..., k$. Therefore, we have
\begin{equation}\label{eq:jacobprojec}
    J_{k, H_k}^M \leq \Pi_{j = 1}^k ||E \nabla z_j(x)|| \leq ||E||^ k \Pi_{j = 1}^k ||\nabla z_j(x)||,
\end{equation}
where $||E||$ denotes the matrix norm which is defined as
\begin{equation*}
    ||E|| = \sup \Big \{||Ey|| \Big | y \in \mathbb R^k, ||y|| = 1 \Big \}.
\end{equation*}
Note that $E$ does not depend on $F$ (or $z_1, ..., z_k$) but only on $T_x M$ or more generally the geometry of $M$ at any point $x$.
From Theorem 3.2 by \citet{Hanin2018ProductsOM} we have, for any fixed $x$,
\begin{equation}\label{eq:nica}
    \mathbb E \Big [\Pi_{j = 1}^k ||\nabla z_j(x)|| \Big ] \leq \Big ( C_{\text{grad}} \Big )^k,
\end{equation}

where,
\begin{equation*}
    C_{\text{grad}} = \sup_z \sup_{x \in \mathbb R^{\nin}} \mathbb E [|| \nabla z(x) ||^{2k} ]^{1/k} \leq C e^{C \sum_{j = 1}^d \frac{1}{n_j}},
\end{equation*}
wherein $C > 0$ depends only on $\mu$ and not on the architecture of $F$ and $n_j$ is the width of the hidden layer $j$.
Let $C_M$ be defined as
\begin{align*}
    C_{M} \defeq \sup \Big \{& C | \text{ there exists a set, S, of non zero } m-k \text{-dimensional Hausdorff measure}\\ 
    &\text{ such that } ||E_x|| \geq C \forall x \in S  \Big \}
\end{align*}
Therefore, combining equations \ref{eq:nica}, \ref{eq:jacobprojec} and result from Theorem 3.2 we have
\begin{align*}
    &\frac{\mathbb E[\vol_{m - k}(\calB_{F, k} \cap M)]}{\vol_{m}(M)} \leq \begin{pmatrix} \text{number of neurons} \\ k \end{pmatrix} (2C_{\text{grad}} C_{\text{bias}} C_{M})^k,
\end{align*}
where the expectation is over the distribution of weights and biases.

\begin{figure}
    \centering
    \subfigure[]{\includegraphics[width=0.45\textwidth]{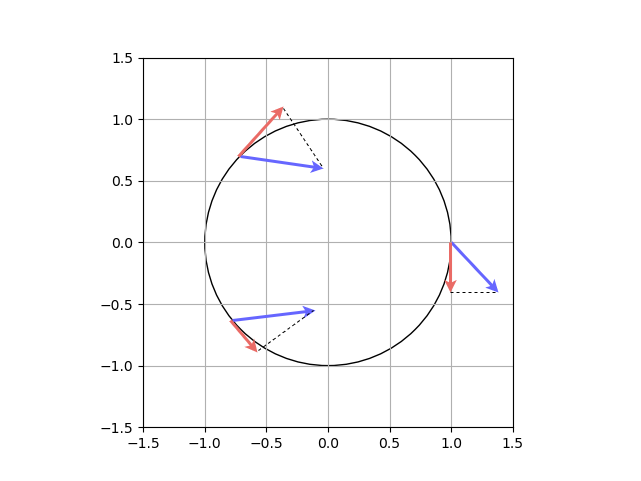}}
    \subfigure[]{\includegraphics[width=0.45\textwidth]{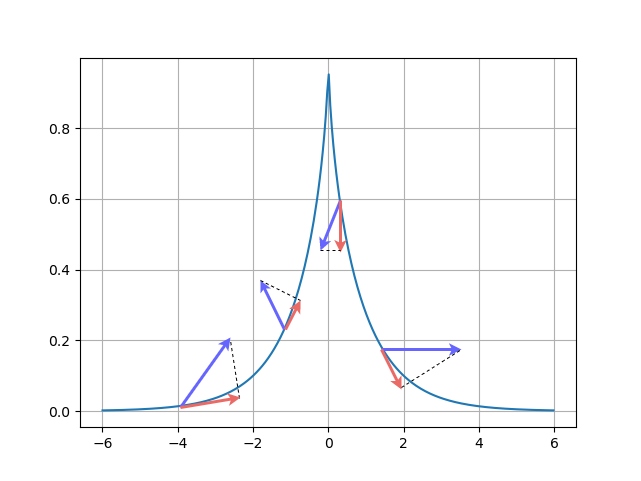}}
    \caption{We illustrate how vectors project differently on tangent planes of two different manifolds: circle (a) and tractrix (b). In case of the tractrix the tangents (and the projection of vectors onto them) are on the inside of the tractrix whereas for the sphere the tangents are always on the outside of the sphere. Since the projections of vectors onto the tangent space are an essential aspect of our proof we end up with the term $C_M$, which quantifies the ``shrinking'' of these vectors upon projection, in the inequalities for Theorems 3.3 and 3.4.}
    \label{fig:projection}
\end{figure}

\section{Proof of Theorem 3.4} \label{app:proof:thm3}

We first prove the following proposition

\begin{proposition} \label{prop:volume}
For a compact $m$-dimensional submanifold $M$ in $\mathbb R^n$, $m, n \geq 1$ and $m < n$ let $S \subseteq \mathbb R^n$ be a compact  fixed continuous piece-wise linear submanifold with finitely many pieces and given any $U > 0$. Let $S_0 = \emptyset$ and let $S_k$ be the union of the interiors of all $k$-dimensional pieces of $S \setminus (S_0 \cup ... \cup S_{k - 1})$. Denote by $T_{\epsilon}$ the $\epsilon$-tubuluar neighbourhood of any $X \subset M$ such that
\begin{equation*}
    T_{\epsilon}(X) = \Big \{ y | d_M(y, X) < \epsilon \text{ and } y \in M\Big \},
\end{equation*}
where $\epsilon \in (0, U)$, $d_M$ is the geodesic distance between the point $y$ and set $X$ on the manifold $M$, we have
\begin{equation*}
    \vol_{m}(T_{\epsilon}(S)) \leq \sum_{k = n - m}^{d} \vol_k(S_k \cap M) \omega_{n - k} \epsilon^{n - k} C_{k, \kappa, U}, 
\end{equation*}
where $C_{k, \kappa, U} > 0$ is a constant that depends on the average scalar curvature $\kappa_{(S_k \cap M)^\perp}$ and $U$, and $\omega_{n - k}$ is the volume of the unit ball in $\mathbb R^{n - k}$.
\end{proposition}
\begin{proof}
Define $d$ to be the maximal dimension of linear pieces in $S$. Let $x \in T_{\epsilon}(X \cap M)$. Suppose $x \notin T_{\epsilon}(X \cap M)$ for all $k = n - m, ..., d -1$. Then the intersection of a geodesic ball of radius $\epsilon$ around $s$ with $S$ is a ball inside $S_d \cap M$. Using the convexity of this ball, with respect to the manifold $M$ \citep{Robbin2011INTRODUCTIONTD}, there exists a point $y$ in $S_d \cap M$ such that the geodesic $\gamma: [0, 1] \to M$ with $\gamma(0) = y$ and $\gamma(1) = x$ is perpendicular to $S_d \cap M$ at $y$. Formally, $T_{S_d \cap M} M$ at $y$ is perpendicular to $\dot{\gamma(0)} \in T_M$ at $y$. Let $B_{\epsilon}(N^*(S_d \cap M))$ be the union of all the $\epsilon$ balls along the fiber of the submanifold  $S_d \cap M$. Therefore, we have
\begin{align}\label{eq:ineqvol}
    \vol_{m}&(T_{\epsilon}(S \cap M) \leq \vol_m(B_{\epsilon}(N^*(S_d \cap M)) + \vol_{m}(T_{\epsilon}(S_{\leq d -1} \cap M)),
\end{align} 
where $S_{\leq d - 1} \defeq \cup_{k = 0}^{d - 1} S_k$. We also note that
\begin{align*}
    \vol_m & (B_{\epsilon}(N^*(S_d \cap M)) = \vol_{m + d - n}(S_d \cap M) \vol_{n - d}(B_{\epsilon}((M \cap S_d)^{\perp})),
\end{align*}
where $B_{\epsilon}((M \cap S_d)^{\perp})$ is the average volume of an $\epsilon$ ball in the submanifold of $M$ orthogonal to $M \cap S_d$. This volume depends on the average scalar curvature, $\kappa_{(M \cap S_d)^{\perp}}$ of the submanifold $(M \cap S_d)^{\perp}$. As shown by \citet{Wan2016GEOMETRICIO}, for a fixed point $x \in (M \cap S_d)^{\perp}$
\begin{align*}
    \vol_{n - d}&(B_{\epsilon}(x, (M \cap S_d)^{\perp})) = \omega_{n - d} \epsilon^{n - d}\Big (1 - \frac{ \kappa(x)_{(M \cap S_d)^{\perp}}}{n - d + 2}\epsilon^2 + O(\epsilon^4) \Big ),
\end{align*}
where $\omega_{n - d}$ is the volume of the unit ball of dimension $n - d$, $B_{\epsilon}(x, (M \cap S_d)^{\perp})$ is the geodesic ball of radius $\epsilon$ in the manifold $(M \cap S_d)^{\perp}$ centered at $x$ and $\kappa_{(M \cap S_d)^{\perp}}(x)$ denotes the scalar curvature at point $x$.
\citet{Gray1974TheVO} provides the second order expansion of the formula above.
Given that $\epsilon \in (0, U)$, for all $k \in \{n - m, n- m + 1, ..., d \}$, then we have a smallest $C_{k, \kappa, U}$ such that
\begin{equation}\label{eq:volbound}
    \vol_{k}(B_{\epsilon}(x, (M \cap S_k)^{\perp})) \leq C_{k , \kappa, U} \epsilon^{k}.
\end{equation}
The above inequality follows from assumption A5.
Using the above inequalities \ref{eq:ineqvol}, \ref{eq:volbound} and repeating the argument $d - 1 -n + m$ times we get the result of the proposition.
\end{proof}

We also note that $C_{k, \kappa, U}$ increases monotonically with $U$, this also follows from the volume being monotonically increasing and positive for $\epsilon > 0$.
Finally, we can now prove Theorem 3.4. Let $x \in M$ be uniformly chosen. Then, for all $\epsilon \in (0, U)$, using Markov's inequality and Proposition \ref{prop:volume}, we have
\begin{align*}
    \mathbb E[\distfun_M(x, B_f \cap M)]  &\geq \epsilon \Pr(\distfun_M(x, B_F \cap M) > \epsilon) \\
    &= \epsilon (1 - \Pr(\distfun_M(x, B_F \cap M) <= \epsilon))\\
    &\geq \epsilon (1 - \sum_{k = \nin - m}^{\nin} \vol_k(S_k \cap M) \omega_{n - k} \epsilon^{\nin - k} C_{\nin - k , \kappa, U} \Big )\\
    & \geq \epsilon (1 - \sum_{k = \nin - m}^{\nin} C_{\nin - k , \kappa, U} (C_{\text{grad}} C_{\text{bias}} C_{M} \epsilon \{\# \text{neurons} \})^k \Big ).
\end{align*}
Note that as we increase $U$ the constants $C_{n - k , \kappa, U}$ increase, although not strictly, for all $k$. 
To find the supremum of the expression on the right hand side, of the last inequality, in $\epsilon \in (0, U)$ we multiply and divide the expression by $C_{\text{grad}} C_{\text{bias}} C_{M} \# \text{neurons} $ to get the polynomial
\begin{equation*}
    p_U(\zeta) = \zeta \Big (1 - \sum_{k = \nin - m}^{\nin} C_{\nin - k , \kappa, U} \zeta^k \Big) ,
\end{equation*}
where $\zeta = \epsilon C_{\text{grad}} C_{\text{bias}} C_{M} \# \text{neurons}$ and $\zeta \in (0, U')$ where $U' = U C_{\text{grad}} C_{\text{bias}} C_{M} \# \text{neurons}$.
Let $d_M$ be the diameter of the manifold $M$, defined by $d_M = \sup_{x, y \in M} \distfun_M(x, y)$.
We assume that $d_M$ is finite. 
Taking the supremum over all $U \in (0, d_M]$ or $U' \in (0, d_M']$, where $d_M' = d_M C_{\text{grad}} C_{\text{bias}} C_{M} \# \text{neurons}$, gives us the constant $C_{M, \kappa}$
\begin{equation*}
    C_{M, \kappa} = \sup_{U' \in (0, d_M']} \{ \sup_{\zeta \in (0, U')} \{ p_U(\zeta) \} \}.
\end{equation*}

Since $d_M$ is finite the constant above exists and is finite.
We make a note on the existence of this constant $C_{M, \kappa}$ in the absence of the constraint that the diameter of manifold $M$ is finite.
As $U$ increases the constants $C_{\nin - k, \kappa, U}$ also increase and are all positive.
The solution for $p_U'(\zeta) = 0, \zeta > 0$, which we denote by $\zeta_U$, is unique and keeps decreasing as $U$ increases.
The uniqueness of the solution follows from the fact that the coefficients $C_{\nin - k, \kappa, U}$ are all positive.
We also note that $p_U(\zeta_U)$ need not be equal to $\sup_{\zeta \in (0, U')} \{ p_U(\zeta) \}$ because $\zeta_U$ need not lie in $(0, U')$.
In all such cases $\sup_{\zeta \in (0, U')} \{ p_U(\zeta) \} = p_U(U')$.
Given the polynomial $p_U (\zeta)$ above if we can assert that there exists a $C_U$, and the corresponding $C_{U'}$, such that for all $U > C_U$, and corresponding $U' > C_{U'}$, we have $\sup_{\zeta \in (0, U')} \{ p_U(\zeta) \} = p_U (\zeta_U) < \infty$ and for all $0 < U \leq C_U$ we have $\sup_{\zeta \in (0, U')} \{ p_U(\zeta) \} = p_U (U') < \infty$.
Therefore, $C_{M, \kappa}$ exists and is finite if the previous assertion holds, proving this assertion is beyond the scope of our current work and particularly challenging.

Finally, taking the average over distribution of weights gives us the inequality
\begin{align*}
    \mathbb E[&\distfun_M(x, B_f \cap M)]  \geq \frac{C_{M, \kappa}}{C_{\text{grad}} C_{\text{bias}} C_{M} \# \text{neurons}},
\end{align*}
where $C_{M, \kappa}$ is a constant which depends on the average scalar curvature of the manifold $M$.
This completes the proof of Theorem 3.4.

\subsection{Variations in Supremum}
\label{subsec:varyninm}

\begin{figure*}
\begin{multicols}{2}
    \centering
    \includegraphics[width=.45\textwidth]{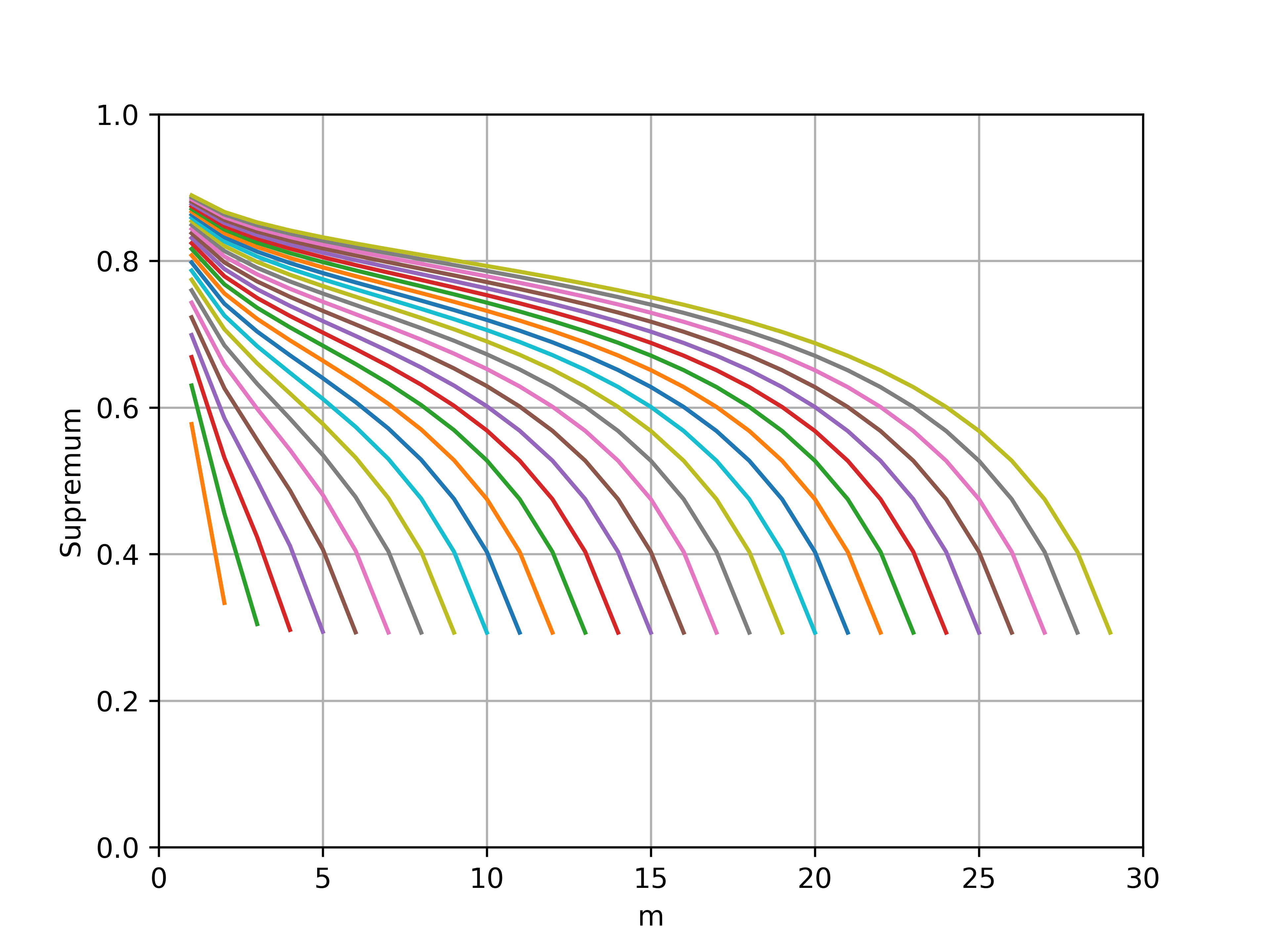}
    \caption{We plot the optima for a simplified polynomial as described in Section \ref{subsec:varyninm}. The individual plots correspond to $\nin$ increasing from $\nin = 2$ to $\nin = 30$ (left to right) with $m$ varying from $1$ to $\nin -1$ on the x-axis.}
    \label{fig:varynm}
    \includegraphics[width=.45\textwidth]{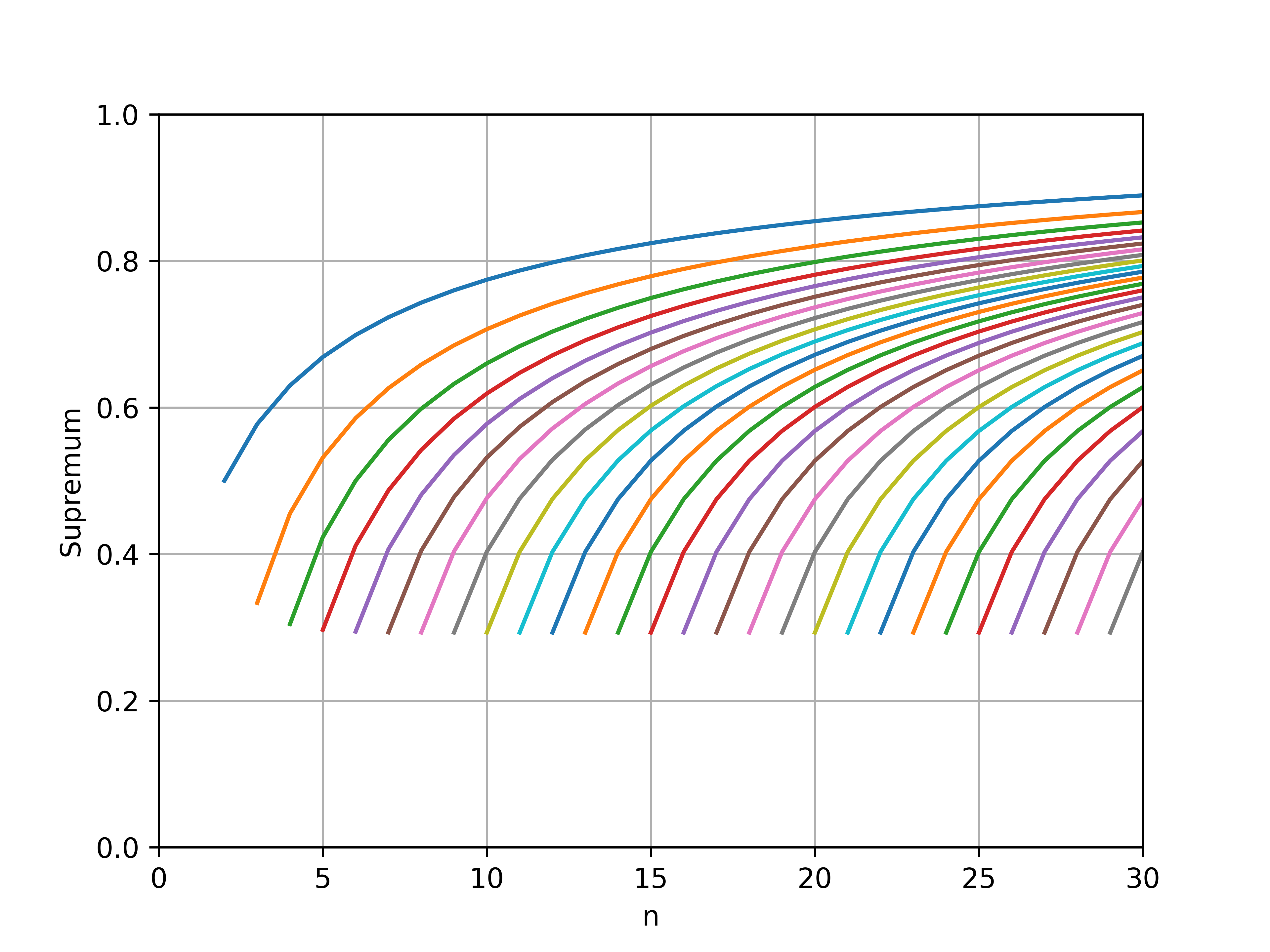}
    \caption{We plot the optima for a simplified polynomial as described in Section \ref{subsec:varyninm}. The individual plots correspond to $m$ increasing from $m = 1$ to $m = 29$ (left to right) with $\nin$ varying from $m + 1$ to $30$ on the x-axis.}
    \label{fig:vary_basem}
\end{multicols}
\end{figure*}

We illustrate the dependence of the the constant $C_{M, \kappa}$ on varying values of $\nin, m$ using a simple example.
We fix the coefficient of the polynomial $p(\zeta)$ to be all 1, this not always the case but we do so to illustrate the relationship between the optima and the exponents for simplest such polynomial:
\begin{equation*}
    p_{\text{simplified}}(\zeta) = \zeta \Big (1 - \sum_{k = \nin -m}^{\nin} \zeta^{k} \Big)
\end{equation*}
We plot the supremums of this simplified polcynomial $C_{\text{simplified}} = \sup_{\zeta \in (0, 1)} p_{\text{simplified}}(\zeta)$ for each $\nin$ from the $\{2, ..., 30\}$ and varying $m$ in Figure \ref{fig:varynm}.
Similarly, we vary $\nin$ with fixed $m$ and report the supremums $C_{\text{simplified}}$ in Figure \ref{fig:vary_basem}.
We notice that for a fixed $\nin$ the supremum decreases with $m$ and for a fixed $m$ the supremum increases with $\nin$.

We programatically calculate the supremum being reported by restricting the domain of  $p_{\text{simplified}}$ to $(0, 1)$.
We solve for the supremum by using the $\texttt{fminbound}$ method from the $\texttt{scipy}$ package \citep{2020SciPy-NMeth}.
The function uses Brent's method \citep{Brent1971AnAW} to find the supremum.

\section{Toy Supervised Learning Problems} \label{app:toyprob}

\begin{figure}
    \centering
    \includegraphics[width=0.55\textwidth]{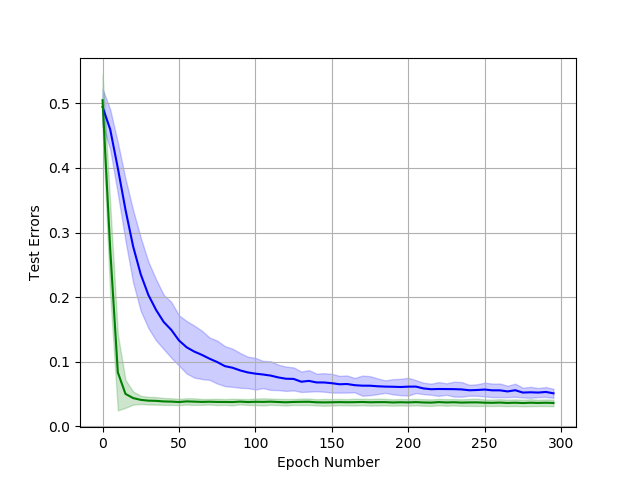}
    \caption{The test errors for the cases where data is sampled from the tractrix (blue) and the circle (green). We see that the tractrix converges slower but the magnitude of the errors remains comparable as training progresses across the two manifolds.}
    \label{fig:test_errs}
\end{figure}
For the two supervised learning tasks with different geometries (tractrix and sphere), we uniformly sample 1000 data points from each 1D manifold to come up with samples of $(x_i, y_i)$ pairs. We then add Gaussian noise to $y$.
We train a DNN with 2 hidden layers, with 10 and 16 neurons in each layer and a single linear output neuron, for a total of 26 neurons with piece-wise linearity, using the PyTorch library.
The optimization is performed using the Adam optimizer \citep{Kingma2015AdamAM} with a learning rate of 0.01.
We ensure a reasonable fit of the model by reducing the test time mean squared error (see Figure \ref{fig:test_errs}).
We then calculate the exact number of linear regions on the respective domains by finding the points where $z(x(t)) = b_z$ for every neuron $z$ and $x$ is on the 1D manifold.
We do this by adding neurons, $z$, one by one at every layer and using the SLSQP \citep{Kraft} to solve for $|z(x(t)) - b_z| = 0$ in $t$ for tractrix and $|z(x(\theta)) - b_z| = 0$ in $\theta$ for the circle.
Note that this methodology can be extended to solve for linear regions of a deep ReLU network for any 1D curve $x(.)$ in any dimension.
We then split a linear region depending on where this solution lies compared to previous layers.
For every epoch, we then uniformly randomly sample points from the 1D manifold, by sampling directly from $\theta$ and $t$, to measure average distance to the nearest linear boundaries.
The experiment was run on CPUs, from training to counting of number of linear regions.
The intel cpus had access to 4 GB memory per core.
A total of, approximately, 24 cpu hours were required for all the experiments in this section.
This was run on an on demand cloud instance.
All implementations are in PyTorch, except for SLQSP for which we used sklearn.
\subsection{Varying Input Dimensions}

The experimental setup, hyperparameters, network architecture, target function and methods are all the same as described for the toy supervised learning problem for the case where the geometry is a sphere.
The only difference is that the input dimension varies, $\nin$.

\section{High Dimensional Dataset} \label{app:hdd}

We utilise the official implementation of pretrained StyleGAN generator to generate curves of images that lie on the manifold of face images.
Specifically, for each curve we sample a random pair of latent vectors: $z_1, z_{2} \in \mathbb R^k$, this gives us the start and end point of the curve using the generator $g(z_1)$ and $g(z_2)$.
We then generate 100 images to approximate a curve connecting the two images on the image manifold in a piece-wise manner.
We do so by taking 100 points on the line connecting $z_1$ and $z_2$ in the latent space that are evenly spaced and generate an image from each one of them.
Therefore, the $i^{\text{th}}$ image is generated as: $x_i = g(((100 - i) \times z_1 + i \times z_2)/100)$, using the StyleGAN generator $g$.
We qualitatively verify the images to ensure that they lie on the manifold of images of faces.
4 examples of these curves, sampled as above, are illustrated in the video here: https://drive.google.com/file/d/1p9B8ATVQGQYoiMh3Q22D-jSaI0USsoNx/view?usp=sharing.

These two constructions allow us to formulate two curves in the high-dimensional setting.
The straight line, with two fixed points $g(z_1)$ and $g(z_2)$, is defined as $x(t) = (1 - t) g(z_1) + t g(z_2)$ with $t \in [0, 1]$.
The approximated curve on the manifold is defined as $x'(t) = (1- t) g(z_i) + t g(z_{i + 1})$ where $i = \texttt{floor}(100t)$.
This once again gives us two curves and we solve for the zeros of $|z(x(t)) - b_z| = 0$ and $|z(x'(t)) - b_z| = 0$ for $t \in [0, 1]$ using SLQSP as described in Appendix \ref{app:toyprob}.

The neural network, used for classification in our MetFaces experiment, is feed forward with ReLU activation. 
There are two hidden layers with 256 and 64 neurons in the first and second layers respectively.
We downsample the images to $128\times 128 \times 3$.
We augment the dataset using random horizontal flips of the images.
All inputs are normalized.
We use a batch size of 32.
The neural network is trained using SGD. 
The learning rate is 0.01 and the momentum is 0.5.
The total time required, for these experiments on MetFaces dataset, was approximately 36 GPU hours on a Titan RTX GPU that has 24 GB memory.
This was run on an on demand cloud instance.
We chose hyperparameters by trial and error, targeting a better fit for the training data for the results reported in Figure 9 of the main body of the paper.

We report further results for density of linear regions with varying hyperparameters in Figure \ref{fig:new_density}.
We also report the training and testing accuracy for the various sets of hyperparameters in Figure \ref{fig:acc_new}.
Note that Figure \ref{fig:acc_new}(a) corresponds to the test and train accuracies on MetFaces reported in the main body of the paper (Figure 9). 
Note all of these results are for the same architecture as described above.

\begin{figure}
    \centering
    \subfigure[LR: 0.025, momentum: 0.5, BS: 64]{\includegraphics[width=0.42\textwidth]{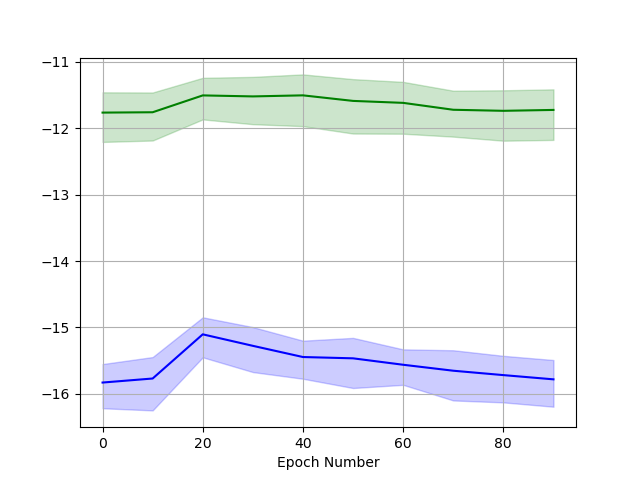}}
    \subfigure[LR: 0.005, momentum: 0.75, BS: 64]{\includegraphics[width=0.42\textwidth]{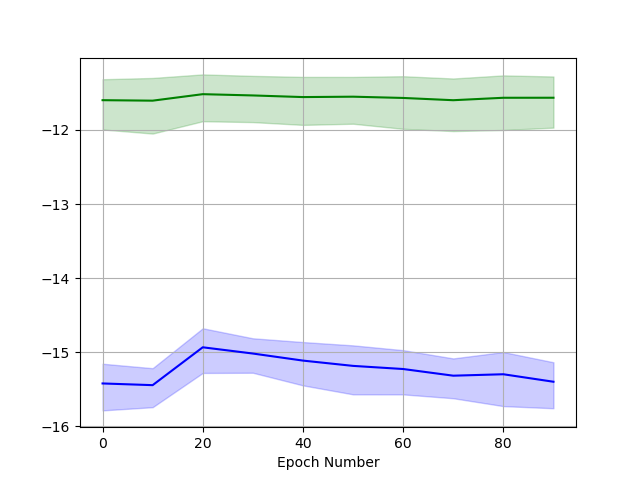}}
    
    \subfigure[LR: 0.01, momentum: 0.75, BS: 128]{\includegraphics[width=0.42\textwidth]{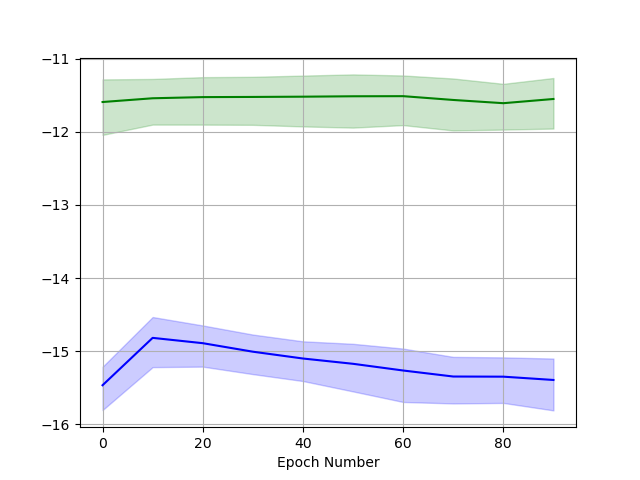}}
    \caption{We report the log density of linear regions for various hyperparameters. Lr refers to the learning rate and BS is the batch size.}
    \label{fig:new_density}
\end{figure}

\begin{figure}
    \centering
    \subfigure[LR: 0.01, momentum: 0.5, BS: 32]{\includegraphics[width=0.42\textwidth]{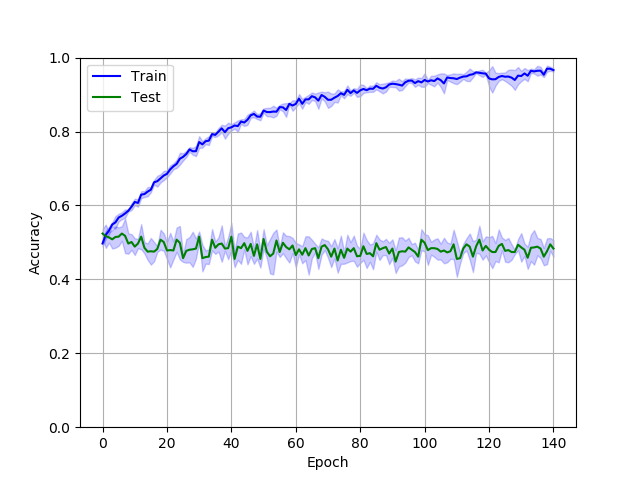}}
    \subfigure[LR: 0.025, momentum: 0.5, BS: 64]{\includegraphics[width=0.42\textwidth]{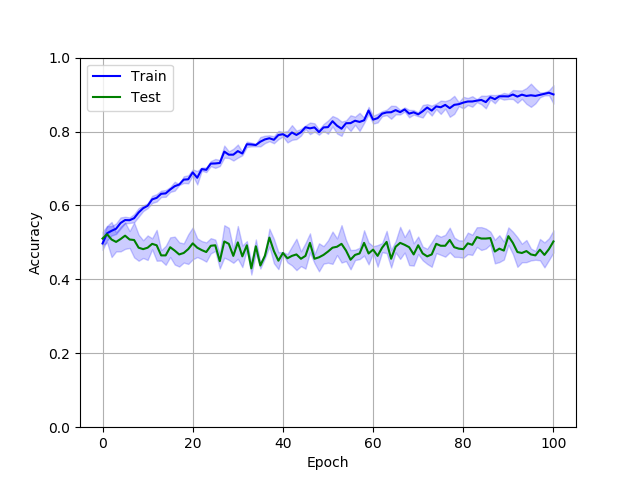}}
    
    \subfigure[LR: 0.005, momentum: 0.75, BS: 64]{\includegraphics[width=0.42\textwidth]{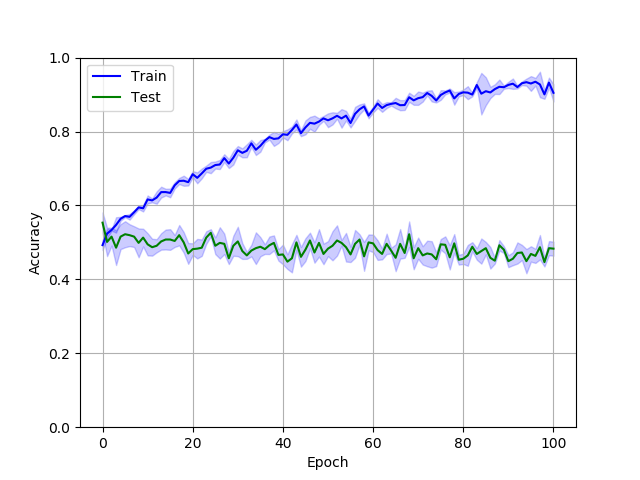}}
    \subfigure[LR: 0.01, momentum: 0.75, BS: 128]{\includegraphics[width=0.42\textwidth]{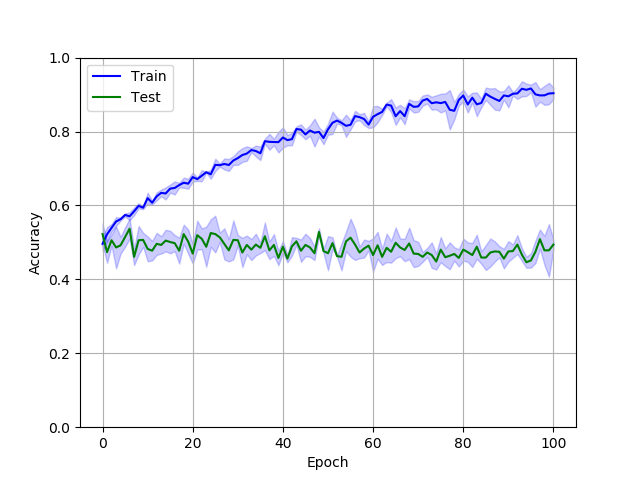}}
    \caption{We report the test and train accuracies across 5 random seeds above.}
    \label{fig:acc_new}
\end{figure}

\end{document}